\documentclass[a4paper,fleqn]{cas-dc}
\usepackage[numbers]{natbib}

\def\tsc#1{\csdef{#1}{\textsc{\lowercase{#1}}\xspace}}
\tsc{WGM}
\tsc{QE}
\tsc{EP}
\tsc{PMS}
\tsc{BEC}
\tsc{DE}

\usepackage{amsmath, amssymb, amsfonts, amsthm}
\newtheorem{thm}{Theorem}
\usepackage{url}
\usepackage{bm}
\usepackage{arydshln}
\usepackage{algorithmic}
\usepackage{algorithm}
\newcommand{\argmax}{\mathop{\rm argmax}\limits}
\newcommand{\largestar}{\bigstar}
\begin{document}

\let\WriteBookmarks\relax
\def\floatpagepagefraction{1}
\def\textpagefraction{.001}
\shorttitle{Grad-CAM for Visualizing Attention Regions of PCA and SVM Layers in Convolutional Neural Networks}
\shortauthors{Y. Omae}

\title[mode = title]{Grad-CAM for Visualizing Attention Regions of PCA and SVM Layers in Convolutional Neural Networks}                    
\author[1]{Yuto Omae}[
                        bioid=1,
                        orcid=0000-0002-5924-6959]\cormark[1]
\ead{oomae.yuuto@nihon-u.ac.jp}
\credit{Conceptualization, Methodology, Software, Validation, Formal analysis, Investigation, Resources, Data Curation, Writing - Original Draft, Writing - Review \& Editing, Visualization, Supervision, Project administration, Funding acquisition}
\affiliation[1]{organization={College of Industrial Technology, Nihon University},
                addressline={1-2-1 Izumi}, 
                city={Narashino},
                postcode={275-8575}, 
                state={Chiba},
                country={Japan}}\cortext[cor1]{Corresponding author}
\author[2]{Hirotaka Takahashi}[
                        bioid=2,
                        orcid=0000-0003-0596-4397]
\ead{hirotaka@tcu.ac.jp}
\credit{Validation, Investigation, Writing - Review \& Editing}
\affiliation[2]{organization={Department of Design and Data Science and Research Center for Space Science, Advanced Research Laboratories, Tokyo City University},
                city={Yokohama},
                postcode={224-8551}, 
                state={Kanagawa},
                country={Japan}}

\begin{abstract}
Convolutional Neural Networks (CNNs) are an effective approach for classification tasks, particularly when the training dataset is large.
Although CNNs have long been considered a black-box classification method, they can be used as a white-box method through visualization techniques such as Grad-CAM.
When the training samples are limited, incorporating a Principal Component Analysis (PCA) layer and/or a Support Vector Machine (SVM) classifier into a CNN can effectively improve the classification performance.
However, a conventional Grad-CAM cannot be directly applied to PCA and/or SVM layers.
Generating attention regions for PCA and/or SVM layers in CNNs is important to facilitate the development of white-box methods.
Therefore, we propose ``PCA-Grad-CAM'', a method for visualizing attention regions in PCA feature vectors, and ``SVM-Grad-CAM'', a method for visualizing attention regions in an SVM classifier layer.
Solving a closed-form Jacobian problem comprising partial derivatives from the last convolutional layer to the PCA and/or SVM layers is necessary to complete the proposed methods analytically.
In this paper, we present the exact closed-form Jacobian and visualization results of the proposed methods applied to several major datasets.
In addition, the insertion and deletion metrics, which are major evaluation metrics in explainable AI, were applied to PCA- and SVM-Grad-CAM.
The results suggest that the proposed method can successfully visualize the attention regions of PCA and SVM. 
\end{abstract}

\begin{graphicalabstract}
\includegraphics[width=\linewidth]{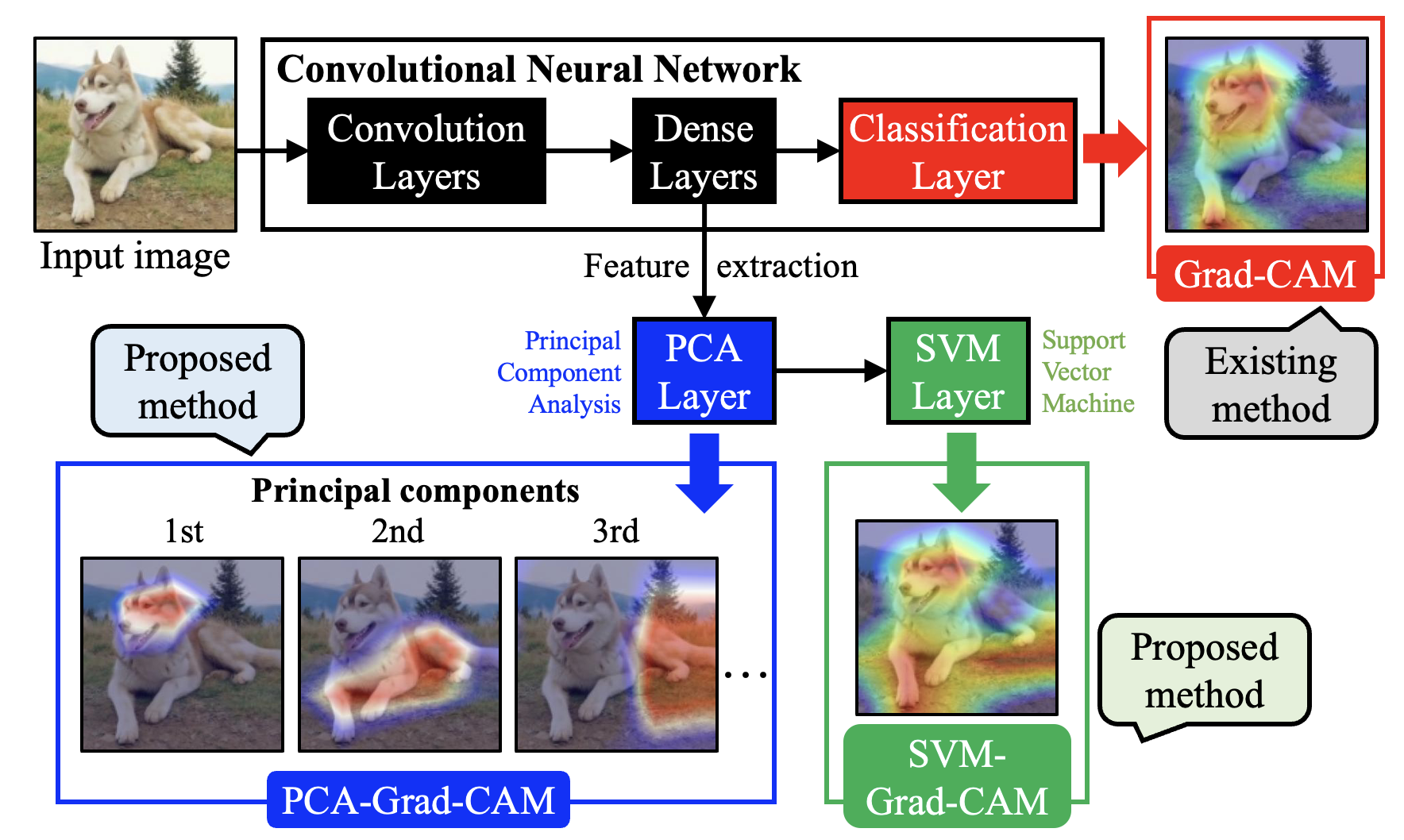}
\end{graphicalabstract}

\begin{highlights}
\item Grad-CAM cannot be directly applied to CNNs augmented with PCA and SVM layers.
\item PCA-Grad-CAM visualizes input contributions to selected principal components.
\item SVM-Grad-CAM explains input contributions to class-specific SVM decision scores.
\end{highlights}

\maketitle

\section{Introduction} \label{sec1_int}
Because Convolutional Neural Networks (CNNs) are an effective approach for various classification tasks, they have been widely used in diverse fields, such as medical diagnosis~\cite{ref_med_cnn1, ref_med_cnn2}, agriculture~\cite{ref_cnn_agri1, ref_cnn_agri2}, industrial applications~\cite{ref_cnn_ind1, ref_cnn_ind2}, and others.
In 2012, AlexNet~\cite{ref_alexnet}, which marked the beginning of the widespread adoption of CNNs, was introduced. However, no technique existed for determining the attention regions of the image that contributed to the predictions of the model.
Grad-CAM, proposed by Selvaraju et al.~\cite{grad_cam_origin} in 2017, enabled the visualization of attention regions in CNNs.
That method is still widely used in various studies, including \cite{ref_use_grad_cam1}, \cite{ref_use_grad_cam2}, and \cite{ref_use_grad_cam3}.
When using CNNs in medical applications, it is recommended to use visualization methods such as Grad-CAM~\cite{ref_med_cam_need}.

\begin{figure*}[pos=tp]
    \begin{center}
        \includegraphics[scale=0.45]{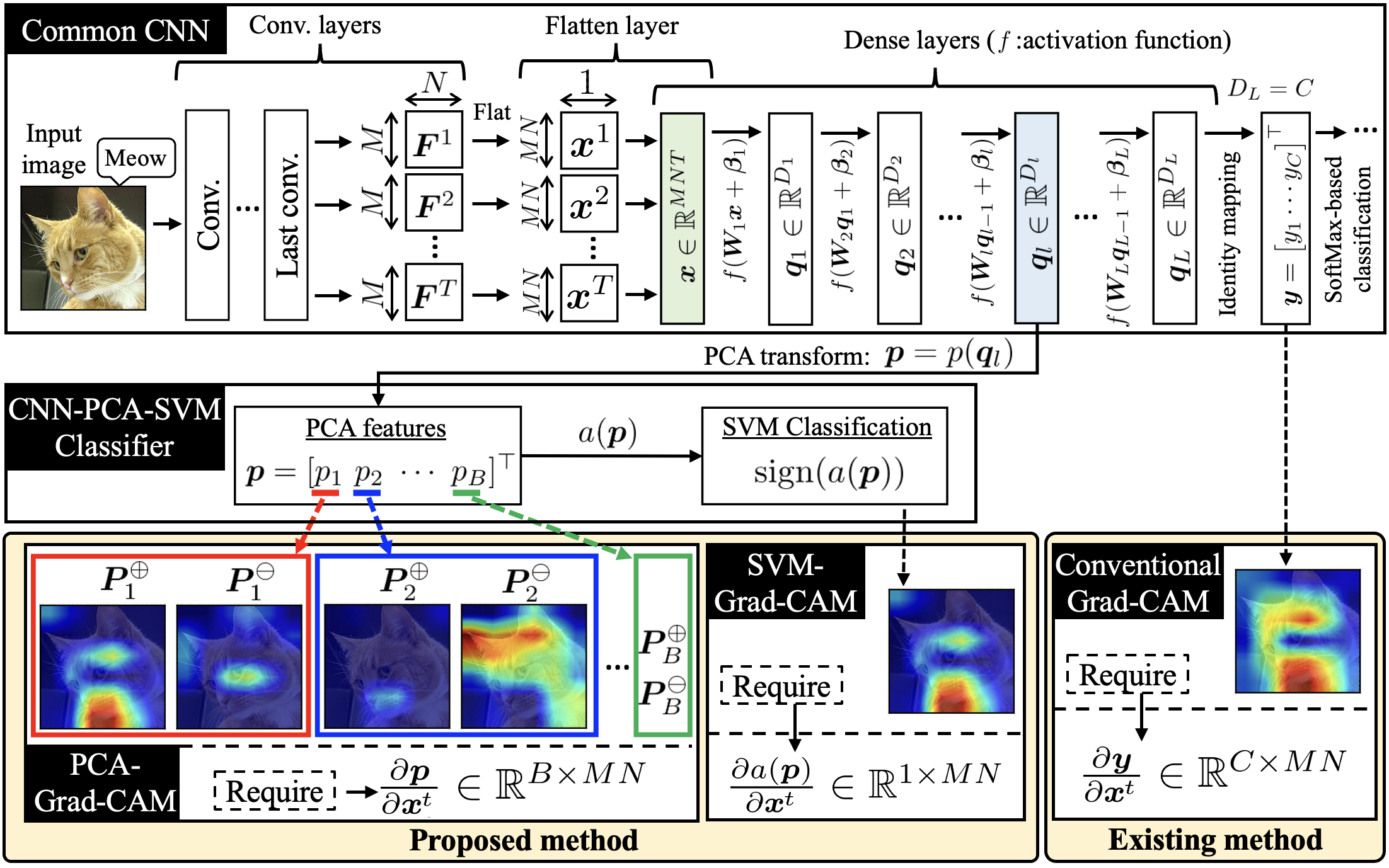}
        \caption{Overview of the existing method (conventional-Grad-CAM \cite{grad_cam_origin}) and the proposed methods (PCA- and SVM-Grad-CAM).}
        \label{fig_ov_view}
    \end{center}
\end{figure*}

Although Grad-CAM is designed for plain CNNs, modern architectures often include specialized components, such as principal component analysis (PCA) and/or support vector machine (SVM) layers, to enhance the performance.
Evidence exists that these ingenuities lead to improved performance.
For instance, Jean et al.~\cite{ref_cnn_svm1} reported an improvement in accuracy from 95.64\% to 98.32\% by implementing an SVM layer into a CNN to classify cocoa beans.
Similar findings have also been reported for tasks such as the detection of cervical cancer cells~\cite{ref_cnn_svm2} and microvascular morphological type recognition~\cite{ref_cnn_svm3}.
Moreover, several studies have implemented both PCA and SVM into CNNs.
Mahmood et al.~\cite{ref_pca_svm1} reported that applying both PCA and SVM to CNNs improved the accuracy from 92.6\% to 94.7\% for the Caltech-101 classification task.
Similar findings have also been reported for tasks such as facial recognition~\cite{ref_pca_svm2}, hand gesture recognition~\cite{ref_pca_svm3}, face verification~\cite{ref_pca_svm4}, oil rig classification~\cite{ref_pca_svm5}, and epilepsy detection~\cite{ref_pca_svm6}.
From these studies, applying both PCA and SVM to CNNs is a popular approach.
However, the conventional Grad-CAM~\cite{grad_cam_origin} cannot be directly applied to these models; that is, they function as black-box models.

Therefore, we propose ``PCA-Grad-CAM'' and ``SVM-Grad-CAM'' to visualize the focus regions of PCA and SVM layers embedded in CNNs, respectively.
An overview comparing the conventional Grad-CAM, PCA-Grad-CAM, and SVM-Grad-CAM is shown in Figure~\ref{fig_ov_view}.
``Conventional-Grad-CAM'' proposed by Selvaraju et al.~\cite{grad_cam_origin} is designed for plain CNNs.
By contrast, PCA-Grad-CAM is a method for visualizing the focus regions of each principal component; that is, if the number of principal components is $B$, then the number of PCA-Grad-CAM visualizations is also $B$.
As shown in Figure~\ref{fig_ov_view}, $\bm{P}^{\omega}_1$, where $\omega \in \{\oplus, \ominus\}$, represents the PCA-Grad-CAM corresponding to the first dimension of the PCA feature vector.
$\bm{P}^\oplus_1$ represents the region that positively influences the first principal component, whereas $\bm{P}^\ominus_1$ represents the region that negatively influences it.
Similarly, $\bm{P}^\omega_2$ where $\omega \in \{\oplus, \ominus\}$ corresponds to the second principal component of the PCA feature vector.
From $\bm{P}^{\oplus}_1$ and $\bm{P}^{\ominus}_2$ in Figure~\ref{fig_ov_view}, it can be inferred that the cat's mouth positively influences the first principal component, whereas the cat's ear negatively influences the second principal component.

Because there are no existing methods for visualizing the focus regions of each principal component in CNNs, the proposed method, PCA-Grad-CAM, could be considered a potentially useful approach.
In contrast, the second proposed method, SVM-Grad-CAM, is designed to visualize the focus region of an SVM layer embedded in CNNs.
As described in the previous paragraph, combining a CNN with an SVM is an effective approach, but such a model functions as a black box.
Because SVM-Grad-CAM aims to address this issue, it may serve as a helpful tool for improving model interpretability.

Additionally, the gradients of certain functions in CNNs, PCA, and SVM are required for the mathematical representation of PCA- and SVM-Grad-CAMs.
Because these are multi-input multi-output functions, their gradients can be represented in Jacobian form.
This study provides exact closed-form expressions for the Jacobians.
Deriving a closed-form solution implies expressing the result explicitly, as if exposing every component of the equation to a full view, without relying on hidden computations or iterative approximations.
Therefore, we provide a theoretical analysis based on these closed-form expressions.

\begin{table*}[pos=tp]
\tabcolsep = 3.5pt
\caption{CAM family comprising various approaches.}
  \label{tab_0}
  \centering
  \begin{tabular}{lcccl}
    \hline\hline
\multicolumn{1}{c}{Method [ref. no.]} & Year &Citation*& Gradient? & Applied cases (detection or prediction target) \\ \hline
CAM~\cite{ref_cam_origin} (Pioneer) & 2015 &13,409& No & Food localization~\cite{ref_apl_cam1}, Diabetic retinopathy~\cite{ref_apl_cam2}, Building defects~\cite{ref_apl_cam3} \\ \hdashline
Grad-CAM~\cite{grad_cam_origin} &2016&31,032& Yes  &Fall armyworm~\cite{ref_use_grad_cam1}, Microscopic algae~\cite{ref_use_grad_cam2}, Gas--liquid jet~\cite{ref_use_grad_cam3}\\
Grad-CAM++~\cite{ref_gradcam_pp_origin} & 2018&3,748& Yes &Lung Cancer~\cite{ref_gcampp_app1}, Brain tumor~\cite{ref_gcampp_app2}, Malignant skin diseases~\cite{ref_gcampp_app3}\\
XGrad-CAM~\cite{ref_xgrad_origin} & 2020&397& Yes&SARS-COVID~\cite{ref_xg_app1}, Brain tumor~\cite{ref_xg_app2}, Specific emitter identification~\cite{ref_xg_app3}\\
Relevance-CAM~\cite{ref_relecam_origin} & 2021&101& Yes  & Compression fracture~\cite{ref_rele_app1}, Small bowel~\cite{ref_rele_app2}, Cities identification~\cite{ref_rele_app3}\\
HiRes-CAM~\cite{ref_hirescom_origin} & 2020&174& Yes &Endoscopic image~\cite{ref_hires_app1}, Insect pests~\cite{ref_hires_app2}, Ship detection~\cite{ref_hires_app3}\\
Layer-CAM~\cite{ref_layercam_origin} &2021 &811& Yes &Aerial scene~\cite{ref_lay_cam_app1}, Pneumoconiosis detection~\cite{ref_lay_cam_app2}, Satellite images~\cite{ref_lay_cam_app3}\\ \hdashline
Score-CAM~\cite{ref_scorecam_origin} & 2020&1,410& No &Plant disease~\cite{ref_score_cam_app1, ref_score_cam_app3}, Skin lesion~\cite{ref_score_cam_app2}, Genotype~\cite{ref_score_cam_app4}\\
Ablation-CAM~\cite{ref_ablationcam_origin} &2020&590& No &Plant disease~\cite{ref_ab_cam_app1}, 
Graphene strength~\cite{ref_ab_cam_app2}, Medicine~\cite{ref_ab_cam_app3, ref_eigen_cam_app1}\\
Shap-CAM~\cite{ref_shap_cam_ori} &2022&54& No & - \\
Shapley-CAM~\cite{ref_shaplaycam_ori} &2025&3& No & -\\
LIFT-CAM~\cite{ref_lift_cam_ori} &2021&125& No &Skin diseases~\cite{ref_lift_cam_app1}\\
Eigen-CAM~\cite{ref_eigen_cam_ori} &2020&545& No &Medical image~\cite{ref_eigen_cam_app1}, Bone loss severity~\cite{ref_eigen_cam_app2}, Power grid fault location~\cite{ref_eigen_cam_app3}\\
    \hline\hline
    \multicolumn{5}{r}{*: Total citations of the papers proposing each method. These values were verified by Google Scholar on 18 July 2025.}
  \end{tabular}
\end{table*}

\section{Related works} 
CAM, proposed by Zhou et al.~\cite{ref_cam_origin}, is a pioneering visualization method for CNNs and has been applied in numerous studies, e.g.,\cite{ref_apl_cam1}, \cite{ref_apl_cam2}, \cite{ref_apl_cam3}.
However, this method can only be applied to a limited set of CNNs that include a global average pooling layer and a classifier layer immediately after the last convolutional layer.
To overcome this limitation, Grad-CAM was proposed by Selvaraju et al.~\cite{grad_cam_origin}.
Because Grad-CAM can be applied to various CNN architectures, it has become a widely used method.
In addition to Grad-CAM, various other CAM-based methods have been proposed, as listed in Table~\ref{tab_0}.
In this section, we first introduce these methods, and then describe the novelty and academic contributions of the proposed approach.

\subsection{Gradient-based CAM}
The methods for visualizing the focus regions of CNNs can be divided into gradient-based and gradient-free approaches.
First, we introduce gradient-based methods, including Grad-CAM++~\cite{ref_gradcam_pp_origin}, XGrad-CAM~\cite{ref_xgrad_origin}, Relevance-CAM~\cite{ref_relecam_origin}, HiRes-CAM~\cite{ref_hirescom_origin}, and Layer-CAM~\cite{ref_layercam_origin}.

Grad-CAM++~\cite{ref_gradcam_pp_origin} applies higher-order derivatives to compute feature map weights.
Grad-CAM++ is superior to Grad-CAM with respect to object detection of a target class.
XGrad-CAM~\cite{ref_xgrad_origin} uses both the activation values and gradients of the feature maps to compute feature map weights.
Grad-CAM uses the sum of the gradients of a feature map as its weight, whereas XGrad-CAM computes the element-wise product of the feature map and its gradients, and then takes the sum as the weight.
Relevance-CAM~\cite{ref_relecam_origin} uses the sum of relevance maps obtained through relevance propagation as the weight for each feature map.
It can visualize not only the last convolutional layer but also all convolutional layers, including shallow layers.
HiRes-CAM~\cite{ref_hirescom_origin} is a CAM method that uses the Hadamard product of the feature map and its gradients.
HiRes-CAM is superior to Grad-CAM in terms of the level of detail in the focus regions.
Layer-CAM~\cite{ref_layercam_origin}, which uses a weighting similar to HiRes-CAM, aims to visualize shallower layers.
As shown in the ``Applied cases'' in Table \ref{tab_0}, the gradient-based CAMs introduced herein have been used in various application studies.

In gradient-based methods, the generated CAM can sometimes change significantly with a slight change in the input.
One method for addressing this issue is to average the CAMs generated from multiple outputs obtained by randomly noised inputs.
This approach is known as SmoothGrad~\cite{ref_smooth_grad}, and it leads to more stable results.
A representative method that incorporates this approach is Smooth Grad-CAM++~\cite{ref_smooth_grad_pp}.

\subsection{Gradient-free CAM} \label{sec2_b}
Next, we introduce gradient-free CAMs, such as Score-CAM~\cite{ref_scorecam_origin}, Ablation-CAM~\cite{ref_ablationcam_origin}, Shap-CAM~\cite{ref_shap_cam_ori}, Shapley-CAM~\cite{ref_shaplaycam_ori}, LIFT-CAM~\cite{ref_lift_cam_ori}, and Eigen-CAM~\cite{ref_eigen_cam_ori}.
Among them, Score-CAM~\cite{ref_scorecam_origin} is the most representative.
The key step in generating Score-CAM is to create a mask image by overlaying a feature map onto the input image.
By inputting this mask image into the trained CNN, the network outputs a class score, which is then used as the weight for the corresponding feature map.
As listed in Table \ref{tab_0}, Score-CAM is the most commonly used method among gradient-free CAMs.
Another approach is Ablation-CAM~\cite{ref_ablationcam_origin}.
The key step is to compute the change in the estimated score by replacing the feature map values with zero or their averages.
The magnitude of the change in the score is used as the weight for the corresponding feature map.
Ablation-CAM is a simple and fast method.

Shap-CAM~\cite{ref_shap_cam_ori} and Shapley-CAM~\cite{ref_shaplaycam_ori} are methods that use the contribution of each pixel to the class prediction, based on Shapley values~\cite{ref_shap}, to determine the weight of each feature map.
Although these methods require high computational costs, they can estimate the contributions more accurately than gradient-based methods.
LIFT-CAM~\cite{ref_lift_cam_ori} employs an approximate approach for computing Shapley values.
These methods are technically interesting because they introduce game theory into CAM.

Additionally, there are CAM methods that summarize all the feature maps.
For instance, Eigen-CAM~\cite{ref_eigen_cam_ori} is generated by applying PCA to the vectorized feature maps.
An essential focus region can be identified by generating only the first principal component.
Note that Eigen-CAM~\cite{ref_eigen_cam_ori} does not reflect the contribution to class prediction.
Although both Eigen-CAM and the proposed PCA-Grad-CAM use PCA, they are fundamentally different.
The main purpose of Eigen-CAM is to extract principal components from all feature maps.
In contrast, the main purpose of PCA-Grad-CAM is to visualize the focus regions corresponding to each principal component in any dense layer.
In other words, PCA-Grad-CAM generates Grad-CAMs for each dimension of the PCA-transformed feature vector in any dense layer.

The gradient-free CAMs introduced in this subsection have been applied in various studies, as shown in the ``Applied cases'' in Table \ref{tab_0}.
However, the application cases of Shapley value-based CAMs are rare.

\subsection{Positioning of this study}
The citation counts of the CAM methods listed in Table~\ref{tab_0} indicate that gradient-based CAMs are more popular than gradient-free CAMs.
In particular, game theory-based CAMs (Shap-CAM~\cite{ref_shap_cam_ori}, Shapley-CAM~\cite{ref_shaplaycam_ori}, and LIFT-CAM~\cite{ref_lift_cam_ori}) have relatively few citations.
Although these methods are interesting, their applications may be limited owing to factors such as implementation difficulty and high computational cost.
In contrast, the most popular method is Grad-CAM, with total citation counts exceeding 30,000.
Therefore, we focused on extending Grad-CAM.

However, as mentioned in Section~\ref{sec1_int}, Grad-CAM cannot be directly applied to PCA and SVM layers integrated into CNNs.
The other CAMs listed in Table~\ref{tab_0} share this limitation, and none of them provides a framework for visualizing each principal component with its positive and negative effects.
In particular, although Eigen-CAM also employs PCA, it is applied to the feature maps rather than to the dense-layer features used for classification, as discussed in Subsection~\ref{sec2_b}.
To address this issue, by extending the theory proposed by Selvaraju et al.~\cite{grad_cam_origin}, we propose two Grad-CAM-based methods, PCA-Grad-CAM and SVM-Grad-CAM, which are applicable to PCA and SVM layers integrated into CNNs.
These methods clearly visualize the focus regions corresponding to each principal component and to the SVM decision.

\section{Preliminary Knowledge} \label{sec3}
This section introduces the background knowledge needed to understand the proposed method. 
The basic mathematical notations are provided in Appendix \ref{app_a}.

\subsection{Conventional CNN}
The CNN model architecture assumed in our study is shown at the top of Figure \ref{fig_ov_view}.
An input image is processed through several convolutional layers, resulting in $T$ feature maps $\bm{F}^1, \cdots, \bm{F}^T$.
Herein, let us denote the $t$-th feature map by $\bm{F}^t = \begin{bmatrix} f^t_{m, n} \end{bmatrix}\in \mathbb{R}^{M \times N}, t \in \mathbb{N}_{\le T}$, which is represented in matrix form, where $M, N$, and $T$ represent the vertical size, horizontal size, and total number of feature maps, respectively.
Using a flatten layer in CNNs, the matrix-form feature map $\bm{F}^t$ is transformed into a vector-form 
\begin{align}
\bm{x}^t = 
\begin{bmatrix}
f^t_{1, 1} \ f^t_{1, 2} \ \cdots \ f^t_{M, N}
\end{bmatrix}^\top = \begin{bmatrix}
x_{1}^t \ x_{2}^t \ \cdots \ x_{MN}^t
\end{bmatrix}^\top \in \mathbb{R}^{MN}. \label{eq_xk}
\end{align}
By applying this transformation to $\forall t \in \mathbb{N}_{\le T}$, the resulting vectors are concatenated into
\begin{align}
\bm{x} = 
\begin{bmatrix}
\bm{x}^1 \\ \vdots \\ \bm{x}^T
\end{bmatrix} \in \mathbb{R}^{MNT}. \label{eq_xvar}
\end{align}
This is generally referred to as a CNN feature vector.
As shown at the top of Figure \ref{fig_ov_view}, the CNN feature vector $\bm{x}$ is used for classification through $L$ fully connected (dense) layers, as follows:
\begin{align}
\bm{q}_1 &= f(\bm{W}_1\bm{x} + \bm{\beta}_1) \in \mathbb{R}^{D_1}, \nonumber \\
\bm{q}_2 &= f(\bm{W}_2\bm{q}_1 + \bm{\beta}_2) \in \mathbb{R}^{D_2}, \nonumber \\
 & \vdots  \nonumber \\
 \bm{q}_l &= f(\bm{W}_l\bm{q}_{l-1} + \bm{\beta}_l) \in \mathbb{R}^{D_l},  \nonumber \\
 & \vdots \nonumber \\
 \bm{q}_{L-1} &= f(\bm{W}_{L-1}\bm{q}_{L-2} + \bm{\beta}_{L-1}) \in \mathbb{R}^{D_{L-1}}, \nonumber \\
 \bm{q}_{L} &= f(\bm{W}_{L}\bm{q}_{L-1} + \bm{\beta}_{L})  \in \mathbb{R}^{D_L}, \nonumber \\
 \bm{y} & = \bm{q}_{L}  \in \mathbb{R}^{C}, \ C=D_L, \label{eq_yql}
\end{align}
where, $\bm{W}_l \in \mathbb{R}^{D_l \times D_{l-1}}$ and $\bm{\beta}_{l} \in \mathbb{R}^{D_l}, D_0 = MNT$.
$\bm{W}_l$ and $\bm{\beta}_l$ are the parameters for the affine transformation of the $l$-th dense layer, and $f$ represents the activation function.
The output vector is $\bm{y} = \begin{bmatrix} y_1 \cdots y_C \end{bmatrix}^\top$, and the predicted class is given by $\mathrm{argmax}_c\ y_c$.

\subsection{PCA layer}
In our study, we assume that PCA is applied to the $l$-th dense layer, $\bm{q}_l = f(\bm{W}_{l}\bm{q}_{l-1} + \bm{\beta}_{l})$.
As indicated by the ``PCA transform'' in Figure \ref{fig_ov_view}, when we denote PCA function by $p(\cdot)$, the PCA feature vector can be represented as $\bm{p} = p(\bm{q}_l)$.
Generally, PCA is used to reduce dimensionality.
Therefore, when we denote the dimension of the PCA feature vector by $B$, the relationship $B \le D_l$ holds.
In other words, the PCA feature vector can be represented as $\bm{p} = \begin{bmatrix} p_1 \ p_2 \ \cdots \ p_B \end{bmatrix}^\top$.

Next, we explain the details of the PCA function $p(\cdot)$.
To obtain this function, a covariance matrix comprising the observation samples is required.
Herein, we assume that $R$ input images are present, and denote vector $\bm{q}_l$ generated from the $r$-th input image by $\bm{q}_l^r$.
In this case, the observation data matrix $\bm{Q}$, its zero-mean version $\bm{Q} - \bm{U}$, the unbiased covariance matrix $\bm{\Sigma}$, and the eigenvalue diagonal matrix $\bm{\Lambda}$ can be represented as
\begin{align}
\bm{Q} &= 
\begin{bmatrix}
\bm{q}_l^1 &
\cdots &
\bm{q}_l^R
\end{bmatrix} \in \mathbb{R}^{D_l \times R}, \nonumber \\
\bm{Q} - \bm{U}&= 
\begin{bmatrix}
\bm{q}_l^1 - \bm{u} 
& \cdots &
\bm{q}_l^R  - \bm{u}
\end{bmatrix} \in \mathbb{R}^{D_l \times R}, \nonumber \\
\bm{\Sigma} &= (\bm{Q} - \bm{U}) (\bm{Q} - \bm{U})^\top /(R-1) \in \mathbb{R}^{D_l \times D_l}, \nonumber \\
\bm{\Lambda} &= \bm{V}^\top \bm{\Sigma} \bm{V} \in \mathbb{R}^{B \times B}, \label{eq_21eig}
\end{align}
where $\bm{u} = \begin{bmatrix}u_1 \ \cdots \ u_{D_l}\end{bmatrix}^\top$.
$u_d$ is the average value of the $d$-th dimension of the vector $\bm{q}_l$.
$B$ is the number of dimensions after dimensionality reduction using PCA.
Additionally, $\bm{\Lambda} = \mathrm{diag}([\lambda_1 \cdots \lambda_B]^\top)$, where $\lambda_1, \cdots, \lambda_B$ are the eigenvalues of $\bm{\Sigma}$ sorted in descending order.
The eigenvectors corresponding to these eigenvalues are denoted as $\bm{v}_1, \cdots, \bm{v}_B$.
The matrix $\bm{V} = \begin{bmatrix}\bm{v}_1 \ \cdots \ \bm{v}_B\end{bmatrix}  \in \mathbb{R}^{D_l \times B}$ is formed by these eigenvectors.
In this case, the process of reducing the dimension of observation matrix $\bm{Q}$ to $B$ is expressed as $\bm{V}^\top (\bm{Q} - \bm{U})$.
In other words, the dimensionality reduction of vector $\bm{q}_l$ generated by an input image and the CNN is expressed as
\begin{align}
\bm{p} = p(\bm{q}_l) = \bm{V}^\top (\bm{q}_l - \bm{u}), \ p: \mathbb{R}^{D_l} \mapsto \mathbb{R}^{B}, \ B \leq D_l. \label{eq_22pca}
\end{align}

\begin{table*}[pos=tp]
\tabcolsep = 1.0pt
  \caption{Closed-form Jacobians for PCA- and SVM-Grad-CAMs, one of our key contributions. Proof $\rightarrow$ Equations \ref{eq_pca_jacc_1}, \ref{eq_svm_jacc_ori}, \ref{eq_qlql1}, \ref{eq_p_q1xt}, \ref{eq_pca_jac_2}, \ref{eq_svm_abstk}, \ref{eq_29}, and \ref{eq_rbf_jacc}.}
  \label{tab_full}
  \centering
  \noindent\makebox[\textwidth]{%
  \begin{tabular}{ccl}
   \toprule
    Act & Kernel & \multicolumn{1}{c}{Closed-form (PCA-Grad-CAM)}\\  \midrule
    Sigm*&-&{\footnotesize$\displaystyle
\frac{\partial  \bm{p}}{\partial  \bm{x}^t} = 
\bm{V}^\top
\left(
\prod_{j=1}^{l-1}
\mathrm{diag} \Bigl( (\bm{1} - f(\bm{\delta}_{l-j+1})) \odot f(\bm{\delta}_{l-j+1}) \Bigr)\bm{W}_{l-j+1}
\right)
\mathrm{diag} \Bigl( (\bm{1} - f(\bm{\delta}_1)) \odot f(\bm{\delta}_1) \Bigr)\bm{W}^t_1
$}\\
    Relu&-&{\footnotesize$\displaystyle
\frac{\partial  \bm{p}}{\partial  \bm{x}^t} = 
\bm{V}^\top
\left(
\prod_{j=1}^{l-1}
\mathrm{diag} ( \bm{1}_{\bm{\delta}_{l-j+1}} )\bm{W}_{l-j+1}
\right)
\mathrm{diag} ( \bm{1}_{\bm{\delta}_1} )\bm{W}^t_1$}\\
\toprule
    Act & Kernel & \multicolumn{1}{c}{Closed-form (SVM-Grad-CAM)}\\  \midrule
    Sigm*&Lin&{\footnotesize$\displaystyle
\frac{\partial  a(\bm{p})}{\partial  \bm{x}^t} = 
\left(
\sum_{i=1}^{I} \alpha_i c_i \bm{p}_i^\top 
\right)
\bm{V}^\top
\left(
\prod_{j=1}^{l-1}
\mathrm{diag} \Bigl( (\bm{1} - f(\bm{\delta}_{l-j+1})) \odot f(\bm{\delta}_{l-j+1}) \Bigr)\bm{W}_{l-j+1}
\right)
\mathrm{diag} \Bigl( (\bm{1} - f(\bm{\delta}_1)) \odot f(\bm{\delta}_1) \Bigr)\bm{W}^t_1
$}\\
    Sigm*&RBF&{\footnotesize$\displaystyle
\frac{\partial  a(\bm{p})}{\partial  \bm{x}^t} = 
2 \gamma \left(
\sum_{i=1}^{I} \alpha_i c_i \exp(-\gamma \| \bm{p}_i - \bm{p}\|_2^2) (\bm{p}_i - \bm{p})^\top 
\right)
\bm{V}^\top
\left(
\prod_{j=1}^{l-1}
\mathrm{diag} \Bigl( (\bm{1} - f(\bm{\delta}_{l-j+1})) \odot f(\bm{\delta}_{l-j+1}) \Bigr)\bm{W}_{l-j+1}
\right)
\mathrm{diag} \Bigl( (\bm{1} - f(\bm{\delta}_1)) \odot f(\bm{\delta}_1) \Bigr)\bm{W}^t_1
$}\\
    Relu&Lin&{\footnotesize$\displaystyle
\frac{\partial  a(\bm{p})}{\partial  \bm{x}^t} = 
\left(
\sum_{i=1}^{I} \alpha_i c_i \bm{p}_i^\top 
\right)
\bm{V}^\top
\left(
\prod_{j=1}^{l-1}
\mathrm{diag} ( \bm{1}_{\bm{\delta}_{l-j+1}} )\bm{W}_{l-j+1}
\right)
\mathrm{diag} ( \bm{1}_{\bm{\delta}_1} )\bm{W}^t_1
$}\\
    Relu&RBF&{\footnotesize$\displaystyle
\frac{\partial  a(\bm{p})}{\partial  \bm{x}^t} = 
2 \gamma \left(
\sum_{i=1}^{I} \alpha_i c_i \exp(-\gamma \| \bm{p}_i - \bm{p}\|_2^2) (\bm{p}_i - \bm{p})^\top 
\right)
\bm{V}^\top
\left(
\prod_{j=1}^{l-1}
\mathrm{diag} ( \bm{1}_{\bm{\delta}_{l-j+1}} )\bm{W}_{l-j+1}
\right)
\mathrm{diag} ( \bm{1}_{\bm{\delta}_1} )\bm{W}^t_1$}
    \\ 
    \bottomrule
\multicolumn{3}{r}{Act: CNN activation function, Kernel: SVM kernel function, Sigm: sigmoid function, Lin: linear kernel, RBF: RBF kernel.}\\
\multicolumn{3}{r}{*: $f(z) = (1+\exp(-z))^{-1}$.}
  \end{tabular}
  }
\end{table*}

\subsection{SVM layer}
As shown in ``SVM Classification'' in Figure \ref{fig_ov_view}, we assume the use of an SVM with the PCA feature vector $\bm{p}$ as input.
The SVM predicted class $s(\bm{p})$ is defined as
\begin{align}
s(\bm{p}) = \mathrm{sign}(a(\bm{p})), \ a(\bm{p}) = \sum_{i=1}^{I} \alpha_i c_i K(\bm{p}_i, \bm{p}) + g, \label{eq_svmsign}
\end{align}
where $\mathrm{sign}(\cdot)$ is a function that returns $1$ if the input value is greater than $0$, and returns $-1$ if the input value is less than $0$.
Moreover, $g$ denotes the bias, $I$ denotes the number of support vectors, $i$ denotes the index of a support vector, $\bm{p}_i$ denotes the $i$-th support vector, $c_i$ denotes the $i$-th actual class, $\alpha_i$ denotes the Lagrange multiplier corresponding to the $i$-th support vector.
Additionally, $K(\bm{p}_i, \bm{p})$ denotes the kernel function.
In this study, we use linear and RBF (Radial Basis Function) kernels, which are expressed as
\begin{align}
&K(\bm{p}_i, \bm{p}) =\begin{cases} 
\bm{p}_i^\top  \bm{p}, & \text{``Linear kernel''}\\
\exp(-\gamma \| \bm{p}_i - \bm{p}\|_2^2), & \text{``RBF kernel''}
\end{cases}, \nonumber \\
&\| \bm{p}_i - \bm{p}\|_2^2 = (\bm{p}_i - \bm{p})^\top (\bm{p}_i - \bm{p}), \ \gamma > 0. \label{eq_kernel_two}
\end{align}

\subsection{Conventional-Grad-CAM}
Next, we describe ``Conventional-Grad-CAM'' as shown in Figure \ref{fig_ov_view}.
This popular CAM method was proposed by Selvaraju et al.~\cite{grad_cam_origin}.
Using a block matrix representation, the Grad-CAM for class $c$ is defined as
\begin{align}
\bm{G}_c &= \mathrm{Relu} \left( \bm{a}_c^\top \bm{F} \right) \in [0, \infty)^{M \times N}, \nonumber \\
\bm{a}_c &= \begin{bmatrix}a_c^1 \\ \vdots \\ a_c^T\end{bmatrix} \in \mathbb{R}^T, \ \bm{F} =\begin{bmatrix}\bm{F}^1 \\ \vdots \\ \bm{F}^T\end{bmatrix} \in \mathbb{R}^{MT \times N}, \label{eq_common_gc}
\end{align}
where $a_c^t \bm{F}^t$ appearing in $\bm{a}_c^\top \bm{F}$ denotes the scalar multiplication of matrix $\bm{F}^t$ by $a_c^t$.
Here, $a_c^t$ represents the weight of the $t$-th feature map $\bm{F}^t$ for class $c$, and is defined as
\begin{align}
a^t_c = \sum_{m=1}^{M} \sum_{n=1}^{N}  \frac{\partial y_c}{\partial f^t_{m,n}} = \sum_{i=1}^{MN} \frac{\partial y_c}{\partial x_i^t}. \label{eq_akc}
\end{align}
Equation~\ref{eq_xk} is used for this transformation.
As summarized in Table \ref{tab_0}, the conventional Grad-CAM method described herein has been applied in various studies.

\section{Proposed method}\label{sec4}
In this section, we explain the details of PCA- and SVM-Grad-CAM by extending the existing Grad-CAM method described in Section \ref{sec3}.
Subsections \ref{sec4pca} and \ref{sec4svm} describe PCA-Grad-CAM and SVM-Grad-CAM, respectively.

\subsection{PCA-Grad-CAM}\label{sec4pca}
The plus-minus sign of the partial derivative $\partial y_c / \partial x_i^t$ appearing in Equation~\ref{eq_akc}, which is required for conventional Grad-CAM, indicates the direction in which the estimated class score changes due to the local effect of a feature map pixel.
Specifically, a positive sign indicates that a pixel has a positive effect on the estimated class, whereas a negative sign indicates a negative effect.
The objective of the conventional Grad-CAM is to clearly highlight the focus region when a CNN predicts a class.
Therefore, only the positive effects are visualized using the rectified linear unit (ReLU) function, as shown in Equation~\ref{eq_common_gc}.
In contrast, both the positive and negative signs of a PCA feature vector are important.
Therefore, visualizing only the positive focus regions is insufficient for PCA-Grad-CAM; that is, visualizing the negative focus regions is also important.
Therefore, PCA-Grad-CAM should include both positive and negative focus regions.

Specifically, we define the PCA-Grad-CAM of the $b$-th principal component as
\begin{align}
\bm{P}_b &= \bm{e}_b^\top \bm{F} \in (-\infty, \infty)^{M \times N}, \nonumber \\
\bm{e}_b &= \begin{bmatrix}e_b^1 &\cdots& e_b^T\end{bmatrix}^\top \in \mathbb{R}^T, \ e^t_b = \sum_{i=1}^{MN} \frac{\partial p_b}{\partial x_i^t}, \label{eq_pca_grad_cam}
\end{align}
where $e^t_b$ denotes the weight of the $t$-th feature map $\bm{F}^t$ for the $b$-th principal component.
When the regions in $\bm{P}_b$ have positive values, the corresponding pixels have a positive effect on the $b$-th principal component, whereas regions with negative values have a negative effect.
If a heat map includes both positive and negative effects, understanding the focus regions of the CNN becomes difficult.
Therefore, we split PCA-Grad-CAM $\bm{P}_b$ into two separate heat maps representing the positive and negative effects.
Specifically, $\bm{P}_b$ is divided into 
\begin{align}
\bm{P}_b^\oplus =  \mathrm{Relu}(\bm{P}_b), \ \bm{P}_b^\ominus = \mathrm{Relu}(-\bm{P}_b). \label{eq_pca_pm}
\end{align}
$\bm{P}_b^{\oplus}$ represents the regions with a positive gradient effect on the $b$-th principal component $p_b$, whereas $\bm{P}_b^{\ominus}$ represents the regions with a negative gradient effect on the same component.
Because it is preferable to compare the heat maps of $\bm{P}_b^{\oplus}$ and $\bm{P}_b^{\ominus}$, we adopt
\begin{align}
\nu_b = \max_{\omega \in \{\oplus, \ominus\}} \ \left\{ \max_{m \in \mathbb{N}_{\le M}, n \in \mathbb{N}_{\le N}} (\bm{P}_b^\omega)_{m, n} \right\} \label{eq_nu}
\end{align}
as the maximum value of the color bar when visualizing the heat maps.
$\nu_b$ is defined as the larger maximum value of $\bm{P}_b^{\oplus}$ and $\bm{P}_b^{\ominus}$.

\begin{figure}[pos=tp]
    \begin{center}
        \includegraphics[scale=0.34]{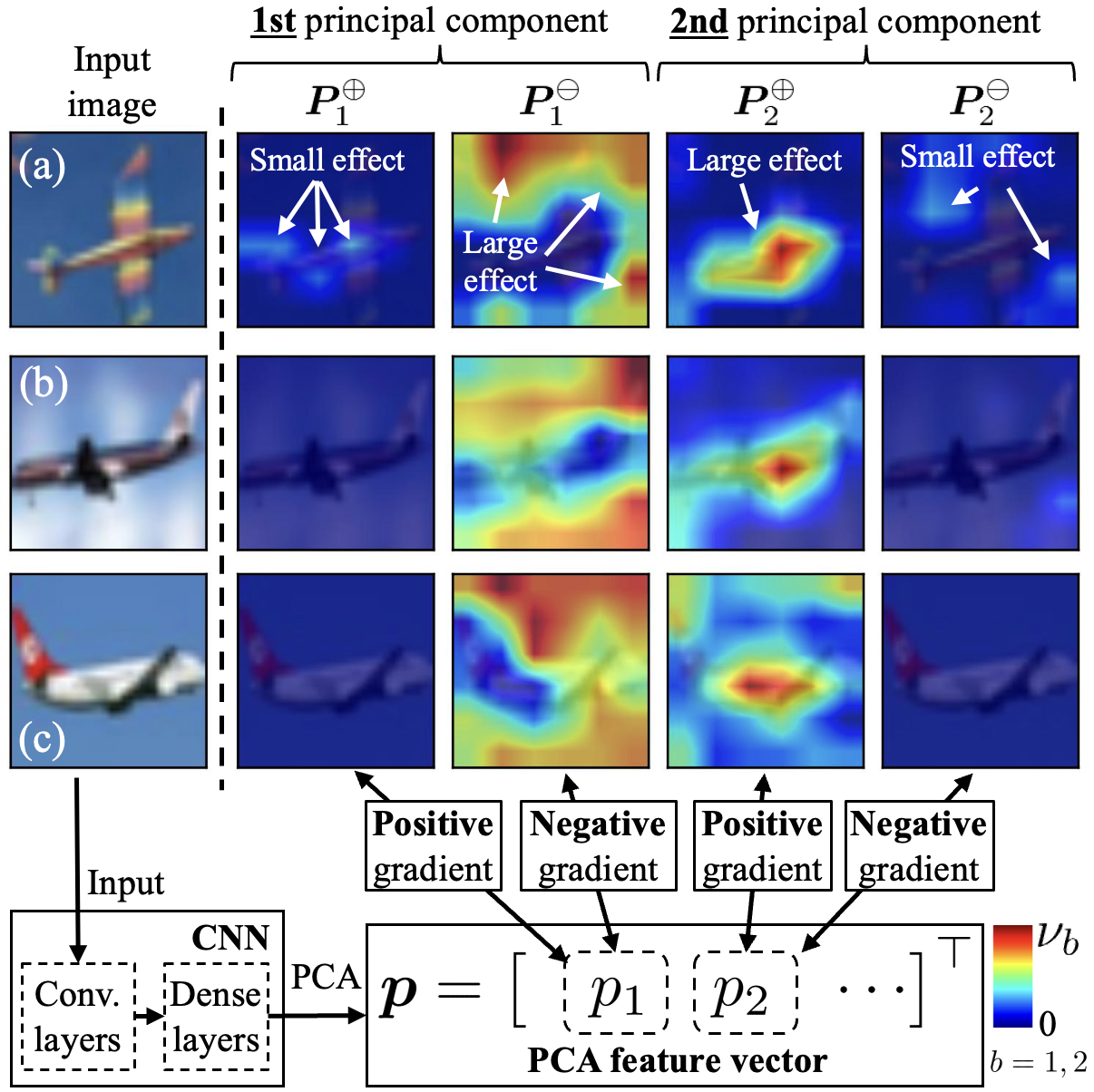}
        \caption{PCA-Grad-CAM visualizations generated from a CNN trained on the CIFAR-10 dataset, as described in Section \ref{sec_vi}.}
        \label{fig_pca0}
    \end{center}
\end{figure}

\begin{figure}[pos=tp]
    \begin{center}
        \includegraphics[scale=0.32]{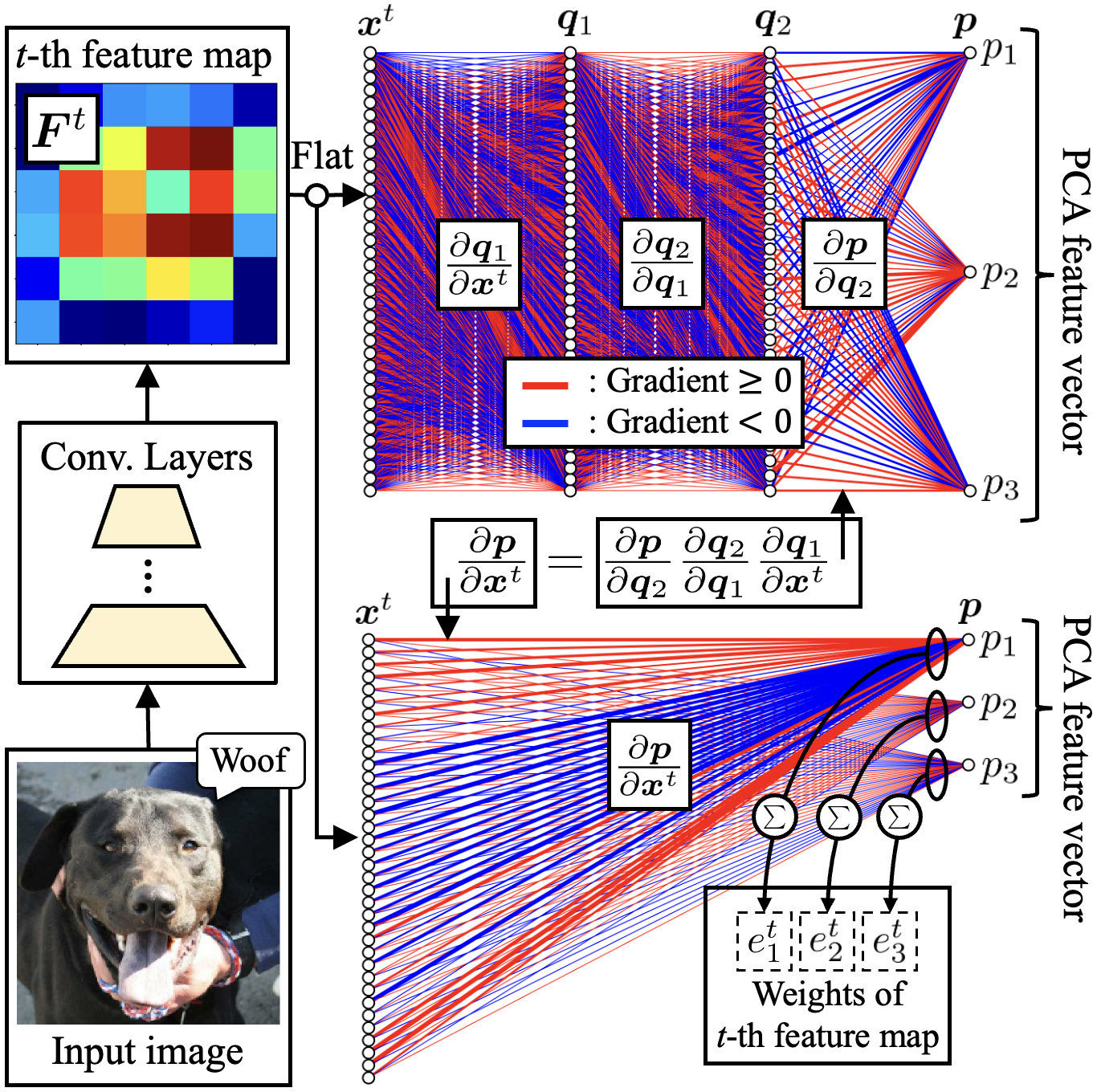}
        \caption{
Visualized Jacobian computed from a CNN trained on the Dog vs. Cat dataset, as described in Section~\ref{sec_vi}.
The red and blue paths represent positive and negative gradients, respectively.
The width of each path represents the absolute value of the gradient.
}
        \label{fig_pca_m1}
    \end{center}
\end{figure}

Figure~\ref{fig_pca0} shows the PCA-Grad-CAMs obtained from a CNN trained on the CIFAR-10 dataset~\cite{ref_cifar10}.
The details of this model are described in Section~\ref{sec_vi}.
$\bm{P}_1^{\oplus}$ shown in Figure~\ref{fig_pca0}(a) is the PCA-Grad-CAM representing the positive gradient effect on the first principal component, indicating that this component focuses slightly on an airplane.
In contrast, $\bm{P}_1^{\ominus}$ shown in the same figure represents the negative gradient effect on the first principal component, indicating that this component strongly focuses on the sky.
As in the comparison shown here, using $\nu_b$ as the maximum color bar allows us to compare $\bm{P}_b^{\oplus}$ and $\bm{P}_b^{\ominus}$.
In summary, $p_1$, the first principal component of the PCA feature vector, was decreased by the sky region.
Next, we focused on the second principal component, as shown in Figure~\ref{fig_pca0} (a).
$\bm{P}_2^{\oplus}$ indicates that the airplane regions were strongly activated.
By contrast, $\bm{P}_2^{\ominus}$ has almost no activated regions.
Therefore, the second principal component $p_2$ of the PCA feature vector can be interpreted as being increased by airplane regions.

The aforementioned tendencies of $\bm{P}_1^{\oplus}$, $\bm{P}_1^{\ominus}$, $\bm{P}_2^{\oplus}$, and $\bm{P}_2^{\ominus}$ are almost the same not only in Figure~\ref{fig_pca0} (a) but also in the cases of (b) and (c).
As illustrated in the aforementioned example, by examining the PCA-Grad-CAMs $\bm{P}_b^{\oplus}$ and $\bm{P}_b^{\ominus}$, we can identify the focus regions of each component of the PCA feature vector.
Notably, Grad-CAM uses partial derivatives as the weights for the feature maps.
Therefore, even if $\bm{P}_b^{\oplus}$ is strongly activated, $p_b$ is not always a positive value.
Similarly, even if $\bm{P}_b^{\ominus}$ is strongly activated, $p_b$ is not always a negative value.

As shown in Equation \ref{eq_pca_grad_cam}, $\partial p_b / \partial x_i^t$ is required to obtain PCA-Grad-CAM.
This variable can be expressed in matrix form as
\begin{align}
\frac{\partial  \bm{p}}{\partial  \bm{x}^t} = 
\begin{bmatrix}
\frac{\partial p_b}{\partial x_i^t}
\end{bmatrix} \in \mathbb{R}^{B \times MN}, \ b \in \mathbb{N}_{\le B}, \ i \in \mathbb{N}_{\le MN}, \nonumber
\end{align}
and is known as the Jacobian.
Using this expression and Equations \ref{eq_xvar}, \ref{eq_yql}, and \ref{eq_22pca}, we obtain 
\begin{align}
\frac{\partial  \bm{p}}{\partial  \bm{x}^t} = 
\overbrace{
\frac{\partial  \bm{p}}{\partial  \bm{q}_{l}}}^{\text{PCA}}
\underbrace{
\frac{\partial  \bm{q}_{l}}{\partial  \bm{q}_{l-1}}
\cdots 
\frac{\partial  \bm{q}_{2}}{\partial  \bm{q}_{1}}
}_{\text{Dense}}
\overbrace{
\frac{\partial  \bm{q}_{1}}{\partial  \bm{x}^t}}^{\text{Flat}}, \ l < L. \label{eq_pca_jacc_1}
\end{align}
The meanings of the left-hand side and the right-hand side of the equation are shown in Figure~\ref{fig_pca_m1}.
This figure assumes that a PCA layer is implemented in the $l$-th dense layer ($l = 2$) of the CNN, as shown in Figure~\ref{fig_ov_view} and Equation \ref{eq_yql}.
The terms $\partial  \bm{p}/ \partial  \bm{q}_{2}$, $\bm{q}_{2} / \partial  \bm{q}_{1}$, and $\partial  \bm{q}_{1}/ \partial  \bm{x}^t$ appearing on the right-hand side of Equation~\ref{eq_pca_jacc_1} represent the Jacobians of directly connected dense layers.
On the other hand, $\partial  \bm{p} / \partial  \bm{x}^t$, the left-hand side of Equation~\ref{eq_pca_jacc_1}, represents the Jacobian expanded using the gradient chain rule.
As shown in the differential graph of $\partial  \bm{p} / \partial  \bm{x}^t$ in Figure~\ref{fig_pca_m1}, we can clearly observe the gradient effect of each pixel in the $t$-th feature map on the PCA feature vector $\bm{p}$.
By summing all paths from $\bm{x}^t$ to $p_b$, the weight of PCA-Grad-CAM, $e^t_b$, is defined as shown in Equation~\ref{eq_pca_grad_cam}.
The closed form of $\partial \bm{p} / \partial \bm{x}^t$ is listed in Table~\ref{tab_full}.
These equations depend on the activation function used in the CNN.
We provide solutions for the Sigmoid and Relu activation functions.
The exact proof is presented in Subsection~\ref{sec_jaccobi}.

\subsection{SVM-Grad-CAM}\label{sec4svm}
Here, we explain a novel method, ``SVM-Grad-CAM,'' as illustrated in Figure~\ref{fig_ov_view}.
Referring to the conventional Grad-CAM defined in Equations \ref{eq_common_gc} and \ref{eq_akc}, we propose SVM-Grad-CAM as 
\begin{align}
\bm{S} &= \mathrm{Relu} \left( \bm{s}^\top \bm{F} \right) \in [0, \infty)^{M \times N}, \nonumber \\
\bm{s} &= \begin{bmatrix}s^1 &\cdots& s^T\end{bmatrix}^\top \in \mathbb{R}^T, \ s^t = \sum_{i=1}^{MN} \frac{\partial a(\bm{p})}{\partial x_i^t}. \label{eq_svm_grad_cam}
\end{align}
The gradients required to compute $s^t$ are summarized as
\begin{align}
\frac{\partial  a(\bm{p})}{\partial  \bm{x}^t} = 
\begin{bmatrix}
\frac{\partial a(\bm{p})}{\partial x_i^t}
\end{bmatrix} \in \mathbb{R}^{1 \times MN}, \ i \in \mathbb{N}_{\le MN}. \nonumber 
\end{align}
By using this expression and Equations \ref{eq_xvar}, \ref{eq_yql}, \ref{eq_22pca}, and \ref{eq_svmsign}, we obtain
\begin{align}
\frac{\partial  a(\bm{p})}{\partial  \bm{x}^t} = 
\underbrace{
\frac{\partial a(\bm{p})}{\partial \bm{p}}}_{\text{SVM}}
\overbrace{
\frac{\partial  \bm{p}}{\partial  \bm{q}_{l}}}^{\text{PCA}}
\underbrace{
\frac{\partial  \bm{q}_{l}}{\partial  \bm{q}_{l-1}}
\cdots 
\frac{\partial  \bm{q}_{2}}{\partial  \bm{q}_{1}}
}_{\text{Dense}}
\overbrace{
\frac{\partial  \bm{q}_{1}}{\partial  \bm{x}^t}}^{\text{Flat}}, \ l < L. \label{eq_svm_jacc_ori}
\end{align}
This closed-form depends on the CNN activation function (Sigmoid or ReLU) and the SVM kernel function (Linear or RBF), as summarized in Table \ref{tab_full}.
The exact proof is provided in Subsection \ref{sec_jaccobi}.

\subsection{Closed-form Jacobians} \label{sec_jaccobi}
To obtain PCA- and SVM-Grad-CAM, the closed-forms of $\partial  \bm{q}_{l} / \partial  \bm{q}_{l-1}$, $\partial  \bm{q}_{1} / \partial  \bm{x}^t$, $\partial  \bm{p} / \partial  \bm{q}_{l}$, and $\partial  a(\bm{p}) / \partial  \bm{x}^t$ appearing in Equations \ref{eq_pca_jacc_1} and \ref{eq_svm_jacc_ori} are required.
These are shown in Table~\ref{tab_full}.
In this subsection, we present the exact proofs of these formulas.

\subsubsection{Jacobian from the dense layer to the next dense layer}
When we represent an affine transform $\bm{W}_{l}\bm{q}_{l-1} + \bm{\beta}_l$ as 
\begin{align}
\bm{\delta}_l = \bm{W}_l \bm{q}_{l-1} + \bm{\beta}_l, \ l > 1, \label{eq_delta_l}
\end{align}
we can represent $\bm{q}_l = f(\bm{\delta}_l)$ as a transformation of the $l$-th dense layer, as shown in Figure \ref{fig_ov_view} and Equation \ref{eq_yql}.
Then, the Jacobian $\partial  \bm{q}_l/\partial  \bm{q}_{l-1}$ can be expressed as
\begin{align}
\frac{\partial  \bm{q}_l}{\partial  \bm{q}_{l-1}} &= 
\frac{\partial  f(\bm{\delta}_l)}{\partial  \bm{\delta}_l}
\frac{\partial  \bm{\delta}_l }{\partial  \bm{q}_{l-1}} = \frac{\partial  f(\bm{\delta}_l)}{\partial  \bm{\delta}_l} \bm{W}_l \in \mathbb{R}^{D_l \times D_{l-1}}. \label{eq_qqjac}
\end{align}
In other words, the Jacobian depends on the activation function $f$.
When we represent $i$-th dimension of $\bm{\delta}_l$ as $(\bm{\delta}_l)_i = \delta_{li}$ and adopt the sigmoid function as the activation function, we obtain
\begin{align}
&f(z) = (1+\exp(-z))^{-1} \Rightarrow \nonumber \\ 
&\left( \frac{\partial  f(\bm{\delta}_l)}{\partial \bm{\delta}_l} \right)_{ij} =
 \frac{\partial  (1+\exp(-\delta_{li}))^{-1}}{\partial  \delta_{lj}} \nonumber \\ &
 =\begin{cases}
 (1-f(\delta_{li}))f(\delta_{li}), & i=j\\
 0, & i \neq j
 \end{cases} \nonumber \\
&  \therefore \ 
 \frac{\partial  f(\bm{\delta}_l)}{\partial \bm{\delta}_l} =\mathrm{diag} \Bigl( (\bm{1} - f(\bm{\delta}_l)) \odot f(\bm{\delta}_l) \Bigr) \in \mathbb{R}^{{D_l} \times {D_l}}. \label{eq_sigm_dif}
\end{align}
Here, $\odot$ denotes the Hadamard product, and $\mathrm{diag}(\bm{z})$ denotes the diagonal matrix of vector $\bm{z}$.
When we adopt the Relu function as the activation function, we obtain
\begin{align}
&f(z) = \mathrm{Relu}(z) \Rightarrow \nonumber \\ 
&\left( \frac{\partial  f(\bm{\delta}_l)}{\partial \bm{\delta}_l} \right)_{ij} =
 \frac{\partial \mathrm{Relu}(\delta_{li})}{\partial  \delta_{lj}}
 =\begin{cases}
 1,& i=j \land \delta_{li} > 0\\
 0,& \text{otherwise}
  \end{cases} \nonumber \\
&  \therefore \ 
 \frac{\partial  f(\bm{\delta}_l)}{\partial \bm{\delta}_l} =\mathrm{diag}( 
\bm{1}_{\bm{\delta}_l}) \in \{0, 1\}^{D_l \times D_l}, \ \bm{1}_{\bm{\delta}_l} \in \{0, 1\}^{D_l},
\label{eq_relu_dif}
\end{align}
where 
\begin{align}
(\bm{1}_{\bm{\delta}_l})_i
 = \begin{cases}
 1,& \delta_{li} > 0\\
 0,& \text{otherwise}
  \end{cases}. \nonumber
\end{align}
Therefore, using Equations \ref{eq_qqjac}, \ref{eq_sigm_dif}, and \ref{eq_relu_dif}, the closed-form Jacobian can be expressed as
\begin{align}
\frac{\partial  \bm{q}_l}{\partial  \bm{q}_{l-1}} = \begin{cases}
\mathrm{diag} \Bigl( (\bm{1} - f(\bm{\delta}_l)) \odot f(\bm{\delta}_l) \Bigr)\bm{W}_l, \ f: \text{``Sigmoid''} \\
\mathrm{diag} ( \bm{1}_{\bm{\delta}_l} )\bm{W}_l, \ f: \text{``Relu''}
\end{cases}.
 \label{eq_qlql1}
\end{align}

\subsubsection{Jacobian from the flatten layer to the dense layer}
Here, we describe the Jacobian $\partial  \bm{q}_{1}/\partial  \bm{x}^t$.
The output $\bm{q}_1 \in \mathbb{R}^{D_1}$ with respect to input $\bm{x}$ is computed as $\bm{q}_1 = f(\bm{W}_1\bm{x} + \bm{\beta}_1)$.
Similar to Equation \ref{eq_delta_l}, we adopt
\begin{align}
\bm{\delta}_1 = \bm{W}_1\bm{x} + \bm{\beta}_1 \in \mathbb{R}^{D_1}. \label{eq_delta_1}
\end{align}
Because $\bm{x}^t$ is used in $\partial  \bm{q}_{1}/\partial  \bm{x}^t$ instead of $\bm{x}$, unlike in Equation \ref{eq_delta_l}, we express the weight matrix $\bm{W}_1$ in terms of 
\begin{align}
\bm{W}_1 = \begin{bmatrix}\bm{W}_1^1 & \cdots & \bm{W}_1^T \end{bmatrix} \in \mathbb{R}^{D_1 \times MNT}, \ \bm{W}_1^t \in \mathbb{R}^{D_1 \times MN}. \nonumber 
\end{align}
By adopting this expression, $\bm{W}_1^t$ corresponds to $\bm{x}^t$ in Equations \ref{eq_xk} and \ref{eq_xvar}.
Specifically, $\bm{\delta}_1$ is expressed as
\begin{align}
\bm{\delta}_1 = \sum_{t=1}^{T} \bm{W}_1^t \bm{x}^t + \bm{\beta}_1 \in \mathbb{R}^{D_1}. \nonumber 
\end{align}
Therefore, we obtain
\begin{align}
\frac{\partial \bm{q}_{1}}{\partial \bm{x}^t} = 
\frac{\partial f(\bm{\delta}_{1})}{\partial \bm{\delta}_{1}}
\frac{\partial \bm{\delta}_{1}}{\partial \bm{x}^t}=
\frac{\partial f(\bm{\delta}_{1})}{\partial \bm{\delta}_{1}} \bm{W}_1^t 
\in \mathbb{R}^{D_1 \times MN}. \nonumber 
\end{align}
Using Equations \ref{eq_sigm_dif} and \ref{eq_relu_dif} we obtain
\begin{align}
\frac{\partial  \bm{q}_1}{\partial  \bm{x}^t} = \begin{cases}
\mathrm{diag} \Bigl( (\bm{1} - f(\bm{\delta}_1)) \odot f(\bm{\delta}_1) \Bigr)\bm{W}^t_1, \ f: \text{``Sigmoid''} \\
\mathrm{diag} ( \bm{1}_{\bm{\delta}_1} )\bm{W}^t_1, \ f: \text{``Relu''}
\end{cases}.
 \label{eq_p_q1xt}
\end{align}

\subsubsection{On the PCA and SVM layers Jacobian}
From Equation \ref{eq_22pca}, the Jacobian of the PCA layer $\partial  \bm{p}/\partial  \bm{q}_{l}$ can be expressed as
\begin{align}
\frac{\partial  \bm{p}}{\partial  \bm{q}_l} = \frac{\partial  \bm{V}^\top (\bm{q}_l - \bm{u})}{\partial  \bm{q}_l} = \bm{V}^\top \in \mathbb{R}^{B \times D_l}. \label{eq_pca_jac_2}
\end{align}
For the SVM layer, from Equation \ref{eq_svmsign}, $\partial  a(\bm{p})/\partial  \bm{p}$ can be expressed as
\begin{align}
\frac{\partial a(\bm{p})}{\partial \bm{p}} &= 
\sum_{i=1}^{I} \alpha_i c_i \frac{\partial  K(\bm{p}_i, \bm{p} )}{\partial \bm{p}} \in \mathbb{R}^{1 \times B}.\label{eq_svm_abstk} 
\end{align}
In other words, the gradient of the kernel function $K(\bm{p}_i, \bm{p})$ is required.
From Equation \ref{eq_kernel_two}, when adopting the linear kernel, the gradient can be expressed as
\begin{align}
K(\bm{p}_i, \bm{p}) = \bm{p}_i^\top  \bm{p} \ \Rightarrow \ 
\frac{\partial  K(\bm{p}_i, \bm{p} )}{\partial \bm{p}} = \bm{p}_i^\top \in \mathbb{R}^{1 \times B}. \label{eq_29}
\end{align}
On the other hand, from Equation \ref{eq_kernel_two}, when the RBF kernel is adopted, the gradient can be expressed as
\begin{align}
&K(\bm{p}_i, \bm{p}) = \exp(-\gamma \| \bm{p}_i - \bm{p}\|_2^2) \Rightarrow \nonumber \\
&\frac{\partial K(\bm{p}_i, \bm{p})}{\partial  \bm{p}} 
= \frac{\partial \exp(-\gamma \| \bm{p}_i - \bm{p}\|_2^2)}{\partial  (-\gamma \| \bm{p}_i - \bm{p}\|_2^2)}
\frac{\partial (-\gamma \| \bm{p}_i - \bm{p}\|_2^2)}{\partial \bm{p}} \nonumber \\
&= 2 \gamma \exp(-\gamma \| \bm{p}_i - \bm{p}\|_2^2) (\bm{p}_i - \bm{p})^\top  \in \mathbb{R}^{1 \times B},  
\label{eq_rbf_jacc}
\end{align}
where $\gamma$ denotes the hyperparameter of the RBF kernel and is related to the decision boundary of the classification.

\begin{figure}[pos=tp]
    \begin{center}
        \includegraphics[scale=0.7]{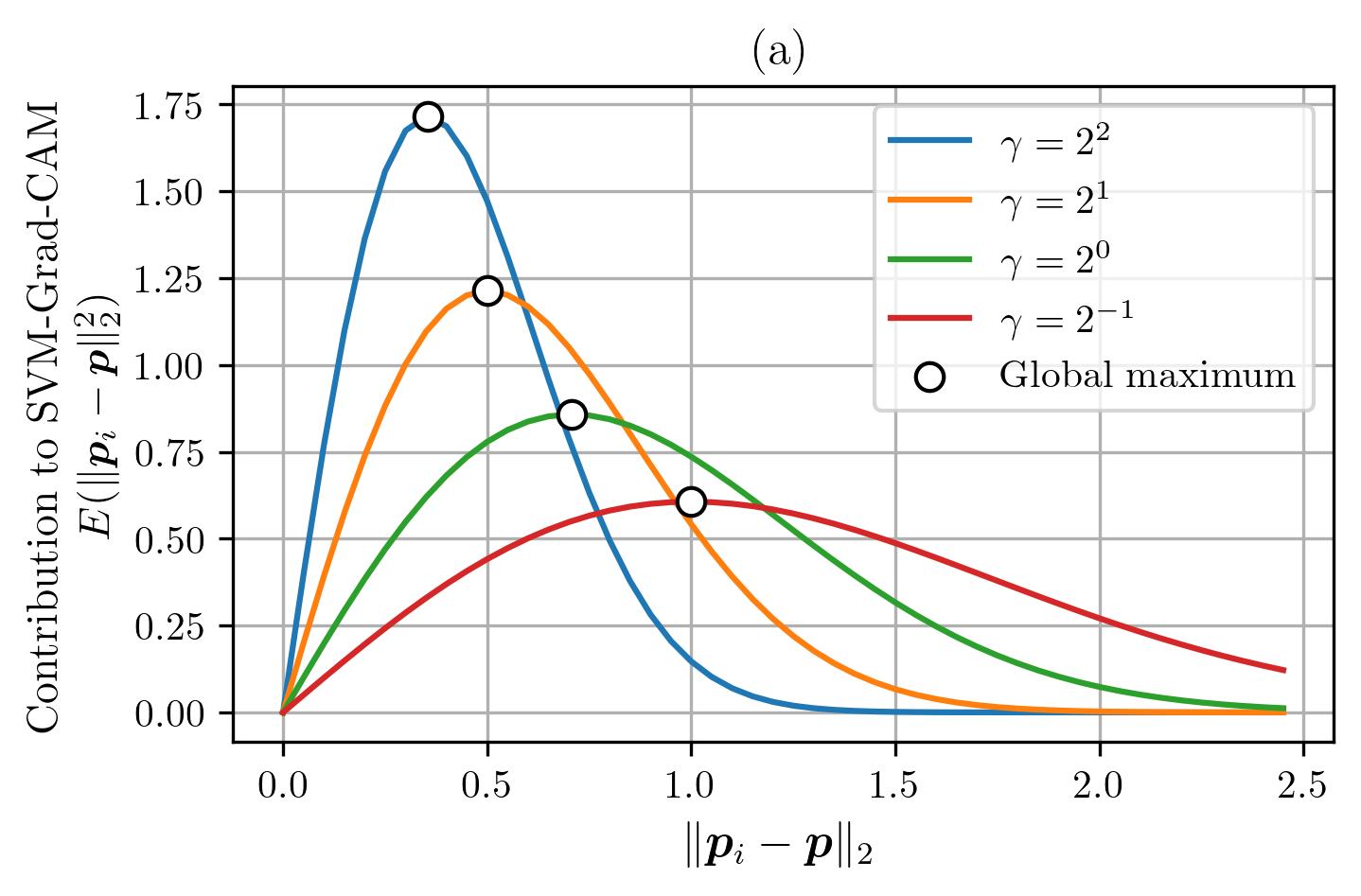}
        \caption{
Effect of the distance between the PCA feature vector $\bm{p}$ and the $i$-th support vector $\bm{p}_i$, $\| \bm{p}_i - \bm{p}\|_2$, on $E(\|\bm{p}_i - \bm{p}\|_2)$.
$E(\|\bm{p}_i - \bm{p}\|_2)$ represents the degree of effect on the SVM-Grad-CAM $\bm{S}$ and is defined in Equation \ref{eq_e_efs}.
From Theorem \ref{th_1}, the distance that maximizes $E(\|\bm{p}_i - \bm{p}\|_2)$ is $\|\bm{p}_i - \bm{p}\|_2 = (2\gamma)^{-1/2}$, and the $\bigcirc$ in the figure represents these locations.
}
        \label{fig_svm1}
    \end{center}
\end{figure}

From these closed-forms, we can understand the characteristics of the SVM-Grad-CAM.
From Equation \ref{eq_29}, when adopting the linear kernel, we can interpret that only the locations of the support vectors $\bm{p}_i$ ($\forall i \in \mathbb{N}_{\le I}$) affect $\partial K(\bm{p}_i, \bm{p})/\partial \bm{p}$.
On the other hand, from Equation \ref{eq_rbf_jacc}, when adopting the RBF kernel, both $\bm{p}$ and $\bm{p}_i$ influence $\partial K(\bm{p}_i, \bm{p})/\partial \bm{p}$.
As shown in Equations \ref{eq_svm_grad_cam}, \ref{eq_svm_jacc_ori}, and \ref{eq_svm_abstk}, $\partial K(\bm{p}_i, \bm{p})/\partial  \bm{p}$ directly affects the SVM-Grad-CAM $\bm{S}$.
Therefore, we define the degree of the effect of the support vector $\bm{p}_i$ on SVM-Grad-CAM as the L2 norm of $\partial K(\bm{p}_i, \bm{p}) / \partial \bm{p}$, as defined in Equation \ref{eq_rbf_jacc}.
In other words, 
\begin{align}
E(\bm{p}_i, \bm{p}) &= \left\| \partial K(\bm{p}_i, \bm{p}) / \partial  \bm{p} \right\|_2 \nonumber \\
&= 2 \gamma \exp(-\gamma \| \bm{p}_i - \bm{p}\|_2^2)  \| \bm{p}_i - \bm{p} \|_2. \label{eq_e_efs}
\end{align}
When $E$ is large, we interpret that the support vector $\bm{p}_i$ strongly affects the SVM-Grad-CAM $\bm{S}$.
Because $E(\bm{p}_i, \bm{p})$ can be regarded as a single-variable function of $\| \bm{p}_i - \bm{p} \|_2$, we express it as
\begin{align}
E(\| \bm{p}_i - \bm{p}\|_2)  = E(\bm{p}_i, \bm{p}).  \label{eq_euc_e}
\end{align}
This indicates that the distance between the support vector $\bm{p}_i$ and the observation PCA vector $\bm{p}$ affects SVM-Grad-CAM $\bm{S}$.
The shape of $E(\bm{p}_i, \bm{p})$ is shown in Figure \ref{fig_svm1} and indicates that when the distance is excessively large, $E(\| \bm{p}_i - \bm{p}\|_2)$ approaches zero.
Moreover, if the distance approaches zero, $E(\| \bm{p}_i - \bm{p}\|_2)$ also approaches zero. 
This visual observation can be proven as the following theorem.
\begin{thm} \label{th_0}
\begin{align}
\lim_{\| \bm{p}_i - \bm{p}\|_2 \rightarrow +0} E(\| \bm{p}_i - \bm{p}\|_2) &= 0, \label{eq_th01}\\
\lim_{\| \bm{p}_i - \bm{p}\|_2 \rightarrow \infty} E(\| \bm{p}_i - \bm{p}\|_2) &= 0 \label{eq_th11}
\end{align}
\end{thm}
\begin{proof}
We omit Equation \ref{eq_th01} because it is trivial.
Because Equation \ref{eq_th11} yields an indeterminate form, we represent Equations \ref{eq_e_efs} and \ref{eq_euc_e} as  
\begin{align}
E(\| \bm{p}_i - \bm{p}\|_2) = \frac{\hat{E}(\| \bm{p}_i - \bm{p}\|_2)}{\check{E}(\| \bm{p}_i - \bm{p}\|_2)} = \frac{2 \gamma  \| \bm{p}_i - \bm{p} \|_2}{\exp(\gamma \| \bm{p}_i - \bm{p}\|_2^2)}. \nonumber
\end{align}
Using Generalized L\textquoteright H\^opital\textquoteright s rule~\cite{ref_lopi}, we obtain
\begin{align}
&\lim_{\| \bm{p}_i - \bm{p}\|_2 \rightarrow \infty} E(\| \bm{p}_i - \bm{p}\|_2) \nonumber \\
&\qquad = \lim_{\| \bm{p}_i - \bm{p}\|_2 \rightarrow \infty} \frac{\mathrm{d} \hat{E}(\| \bm{p}_i - \bm{p}\|_2) / \mathrm{d} \| \bm{p}_i - \bm{p}\|_2}{\mathrm{d} \check{E}(\| \bm{p}_i - \bm{p}\|_2)/\mathrm{d} \| \bm{p}_i - \bm{p}\|_2}\nonumber \\
&\qquad = \lim_{\| \bm{p}_i - \bm{p}\|_2 \rightarrow \infty} \frac{1}{\| \bm{p}_i - \bm{p}\|_2 \exp(\gamma \| \bm{p}_i - \bm{p}\|_2^2)} = 0. \nonumber
\end{align}
\end{proof}
This theorem indicates that when the distance is either excessively short or excessively long, $E(\| \bm{p}_i - \bm{p}\|_2)$, the effect on SVM-Grad-CAM $\bm{S}$, converges to zero.
Thus, maintaining a moderate distance is important.
Therefore, we consider the distance $\| \bm{p}_i - \bm{p}\|_2$ that maximizes $E(\| \bm{p}_i - \bm{p}\|_2)$.
This can be explained by the following theorem.
\begin{thm} \label{th_1}
\begin{align}
\argmax_{\| \bm{p}_i - \bm{p}\|_2 \in [0, \infty)} E(\| \bm{p}_i - \bm{p}\|_2) &= \frac{1}{\sqrt{2\gamma}}\label{eq_argmax}\\
\max_{\| \bm{p}_i - \bm{p}\|_2  \in [0, \infty)} E(\| \bm{p}_i - \bm{p}\|_2) &= \sqrt{2\gamma} \exp\left(-\frac{1}{2}\right) \label{eq_max}
\end{align}
\end{thm}
\begin{proof}
The first- and second-order derivatives of $E(\| \bm{p}_i - \bm{p}\|_2)$ with respect to distance can be expressed as
\begin{align}
&\frac{\mathrm{d} E(\| \bm{p}_i - \bm{p}\|_2)}{\mathrm{d} \|\bm{p}_i - \bm{p}\|_2} = 2\gamma (1 - 2 \gamma \|\bm{p}_i - \bm{p}\|_2^2) \exp(-\gamma \| \bm{p}_i - \bm{p}\|_2^2), \nonumber \\
&\frac{\mathrm{d}^2 E(\| \bm{p}_i - \bm{p}\|_2)}{\mathrm{d} (\|\bm{p}_i - \bm{p}\|_2)^2} = \nonumber \\ &-4\gamma^2 \|\bm{p}_i - \bm{p}\|_2 (3 - 2\gamma \|\bm{p}_i - \bm{p}\|_2^2) \exp(-\gamma \| \bm{p}_i - \bm{p}\|_2^2). \nonumber 
\end{align}
When the distance is $1/\sqrt{2\gamma}$, we obtain
\begin{align}
\left. \frac{\mathrm{d} E(\| \bm{p}_i - \bm{p}\|_2)}{\mathrm{d} \|\bm{p}_i - \bm{p}\|_2}\right|_{\|\bm{p}_i - \bm{p}\|_2 = \frac{1}{\sqrt{2 \gamma}}} &= 0, \nonumber \\
\left. \frac{\mathrm{d}^2 E(\| \bm{p}_i - \bm{p}\|_2)}{\mathrm{d} (\|\bm{p}_i - \bm{p}\|_2)^2} \right|_{\|\bm{p}_i - \bm{p}\|_2 = \frac{1}{\sqrt{2 \gamma}}} &= -4\sqrt{2} \gamma^{\frac{3}{2}} \exp\left(-\frac{1}{2}\right) < 0. \nonumber
\end{align}
Therefore, $\|\bm{p}_i - \bm{p}\|_2 = 1/\sqrt{2 \gamma}$ is a critical point of $E(\| \bm{p}_i - \bm{p}\|_2)$ and corresponds to a local maximum.
Moreover, because
\begin{align}
\frac{\mathrm{d} E(\| \bm{p}_i - \bm{p}\|_2)}{\mathrm{d} \|\bm{p}_i - \bm{p}\|_2} 
\begin{cases}
>0, &  0 \leq \|\bm{p}_i - \bm{p}\|_2 < \frac{1}{\sqrt{2 \gamma}}\\
<0, &  \|\bm{p}_i - \bm{p}\|_2 > \frac{1}{\sqrt{2 \gamma}}
\end{cases} \nonumber
\end{align}
is satisfied, $E(\| \bm{p}_i - \bm{p}\|_2)$ is a monotonically increasing function when the distance is shorter than the critical point.
By contrast, $E(\| \bm{p}_i - \bm{p}\|_2)$ is a monotonically decreasing function when the distance is longer than the critical point.
Therefore, the critical point $\|\bm{p}_i - \bm{p}\|_2 = 1/\sqrt{2 \gamma}$ corresponds to the global maximum of $E(\| \bm{p}_i - \bm{p}\|_2)$, which implies that Equation \ref{eq_argmax} is satisfied.
Moreover, substituting this into $E(\| \bm{p}_i - \bm{p}\|_2)$ yields Equation \ref{eq_max}.
\end{proof}
This theorem indicates that when the distance between the $i$-th support vector $\bm{p}_i$ and the observed PCA feature vector $\bm{p}$ is $1/\sqrt{2 \gamma}$, the effect of $\bm{p}_i$ on SVM-Grad-CAM $\bm{S}$ is maximized.
$\bigcirc$ in Figure \ref{fig_svm1} indicates the location.
From Theorem \ref{th_1} and Figure \ref{fig_svm1}, we can understand that the SVM hyperparameter $\gamma$ directly affects the SVM-Grad-CAM $\bm{S}$.
Specifically, when $\gamma$ is small, the distance that maximizes the effect becomes long, whereas when $\gamma$ is large, it becomes short.

\begin{figure}[pos=tp]
\begin{algorithm}[H]
  \caption{Procedure for computing Jacobians}
  \label{alg1}
  \begin{algorithmic}[1]
  \REQUIRE 
  Image: $X$,
  CNN: $\Upsilon$, 
  activation: $f \in \{\text{``Sigmoid''}, \text{``Relu''} \}$,
  layer number where PCA is applied: $l$,
  number of principal components: $B$, 
  covariance: $\bm{\Sigma}$, 
  average: $\bm{u}$, 
  kernel: $k \in \{\text{``Linear''}, \text{``RBF''} \}$,
  SVM params. $\{(\alpha_i, c_i, \bm{p}_i) \mid i \in \mathbb{N}_{\leq I}\}$ and $\gamma$ 
  \ENSURE $\left\{\frac{\partial  \bm{p}}{\partial  \bm{x}^t} \mid t \in \mathbb{N}_{\le T}\right\}$, $\left\{\frac{\partial  a(\bm{p})}{\partial  \bm{x}^t} \mid t \in \mathbb{N}_{\le T}\right\}$, $\bm{F}$
  \STATE \#\#\# Feature extraction from the input image $X$ \#\#\#
  \STATE $\{\bm{F}^1, \cdots, \bm{F}^T\} \leftarrow \mathrm{GetFeatureMap}(\Upsilon, X)$
  \STATE $\{\bm{x}^1, \cdots, \bm{x}^T\} \leftarrow \mathrm{Flatten}(\{\bm{F}^1, \cdots, \bm{F}^T\})$, Eq.\ref{eq_xk}
  \STATE $\bm{x} \leftarrow \begin{bmatrix}(\bm{x}^1)^\top \cdots (\bm{x}^T)^\top\end{bmatrix}^\top$, Eq. \ref{eq_xvar}
  \STATE $\bm{F} \leftarrow \begin{bmatrix}(\bm{F}^1)^\top  \cdots (\bm{F}^T)^\top\end{bmatrix}^\top$, Eq. \ref{eq_common_gc}
  \STATE $\{(\bm{W}_1, \bm{\beta}_1), \cdots, (\bm{W}_l, \bm{\beta}_l)\} \leftarrow \mathrm{GetParams}(\Upsilon)$, Eq.\ref{eq_yql}
  \STATE \#\#\# Get PCA vector $\bm{p}$ \#\#\#
  \STATE $\{\bm{v}_1, \cdots, \bm{v}_B\} \leftarrow \mathrm{GetEigenVector}(\bm{\Sigma}, B)$, Eq.\ref{eq_21eig}
  \STATE $\bm{V} \leftarrow \begin{bmatrix}\bm{v}_1 &\cdots &\bm{v}_B\end{bmatrix}$, Eq.\ref{eq_21eig}
  \STATE $\bm{q}_1 \leftarrow f(\bm{W}_1\bm{x} + \bm{\beta}_1)$, Eq.\ref{eq_yql}
  \FOR{$j \leftarrow 2$ \TO $l$}
  \STATE $\bm{q}_j \leftarrow f(\bm{W}_j\bm{q}_{j-1} + \bm{\beta}_j)$, Eq.\ref{eq_yql}
  \ENDFOR
  \STATE $\bm{p} \leftarrow \bm{V}^\top (\bm{q}_l - \bm{u})$, Eq.\ref{eq_22pca}
  \STATE \#\#\# Jacobian for the SVM layer \#\#\#
  \IF{$k = \text{``Linear''}$}
  \STATE $\frac{\partial a(\bm{p})}{\partial \bm{p}} \leftarrow \sum_{i=1}^{I} \alpha_i c_i \bm{p}_i^\top$, Eqs.\ref{eq_svm_abstk}, \ref{eq_29}
  \ELSIF{$k = \text{``RBF''}$}
  \STATE $\frac{\partial a(\bm{p})}{\partial \bm{p}} \leftarrow  2 \gamma \sum_{i=1}^{I} \alpha_i c_i \exp(-\gamma \| \bm{p}_i - \bm{p}\|_2^2) (\bm{p}_i - \bm{p})^\top$, Eqs.\ref{eq_svm_abstk}, \ref{eq_rbf_jacc}
  \ENDIF
  \STATE \#\#\# Jacobian for the PCA layer \#\#\#
  \STATE $\frac{\partial \bm{p}}{\partial  \bm{q}_l} \leftarrow \bm{V}^\top$, Eq.\ref{eq_pca_jac_2}
  \STATE \#\#\# Jacobians for the dense layers \#\#\#
  \FOR{$j \leftarrow 2$ \TO $l$}
  \STATE $\bm{\delta}_j \leftarrow \bm{W}_j\bm{q}_{j-1} + \bm{\beta}_j$, Eq.\ref{eq_delta_l}
  \IF{$f = \text{``Sigmoid''}$}
  \STATE $\frac{\partial  \bm{q}_j}{\partial  \bm{q}_{j-1}} \leftarrow \mathrm{diag}\Bigl( (\bm{1} - f(\bm{\delta}_j)) \odot f(\bm{\delta}_j) \Bigr)\bm{W}_j$, Eq.\ref{eq_qlql1}
  \ELSIF{$f = \text{``Relu''}$}
  \STATE $\frac{\partial  \bm{q}_j}{\partial  \bm{q}_{j-1}} \leftarrow \mathrm{diag} ( \bm{1}_{\bm{\delta}_j} )\bm{W}_j$, Eq.\ref{eq_qlql1}
  \ENDIF
  \ENDFOR
  \STATE \#\#\# Jacobians for PCA- and SVM-Grad-CAM \#\#\#
  \FOR{$t \leftarrow 1$ \TO $T$}
  \STATE  $\bm{\delta}_1 \leftarrow \bm{W}_1\bm{x} + \bm{\beta}_1$, Eq.\ref{eq_delta_1}
  \IF{$f = \text{``Sigmoid''}$}
  \STATE $\frac{\partial  \bm{q}_1}{\partial  \bm{x}^t} \leftarrow \mathrm{diag}\Bigl( (\bm{1} - f(\bm{\delta}_1)) \odot f(\bm{\delta}_1) \Bigr) \bm{W}^t_1$, Eq.\ref{eq_p_q1xt}
  \ELSIF{$f = \text{``Relu''}$}
  \STATE $\frac{\partial  \bm{q}_1}{\partial  \bm{x}^t} \leftarrow \mathrm{diag} ( \bm{1}_{\bm{\delta}_1} )\bm{W}^t_1$, Eq.\ref{eq_p_q1xt}
  \ENDIF
  \STATE $\frac{\partial  \bm{p}}{\partial  \bm{x}^t} \leftarrow
  \frac{\partial  \bm{p}}{\partial  \bm{q}_{l}}
  \frac{\partial  \bm{q}_{l}}{\partial  \bm{q}_{l-1}}
  \cdots 
  \frac{\partial  \bm{q}_{2}}{\partial  \bm{q}_{1}}
  \frac{\partial  \bm{q}_{1}}{\partial  \bm{x}^t}$, Eq.\ref{eq_pca_jacc_1}
  \STATE $\frac{\partial  a(\bm{p})}{\partial  \bm{x}^t} \leftarrow \frac{\partial a(\bm{p})}{\partial \bm{p}}
  \frac{\partial  \bm{p}}{\partial  \bm{q}_{l}}
  \frac{\partial  \bm{q}_{l}}{\partial  \bm{q}_{l-1}}
  \cdots 
  \frac{\partial  \bm{q}_{2}}{\partial  \bm{q}_{1}}
  \frac{\partial  \bm{q}_{1}}{\partial  \bm{x}^t}$, Eq.\ref{eq_svm_jacc_ori}
  \ENDFOR
  \RETURN $\left\{\frac{\partial  \bm{p}}{\partial  \bm{x}^t} \mid t \in \mathbb{N}_{\le T}\right\}$, 
  $\left\{\frac{\partial  a(\bm{p})}{\partial  \bm{x}^t} \mid t \in \mathbb{N}_{\le T}\right\}$, $\bm{F}$
   \end{algorithmic}
\end{algorithm}
\end{figure}

\begin{figure}[pos=tp]
\begin{algorithm}[H]
  \caption{Procedure for PCA- and SVM-Grad-CAM}
  \label{alg2}
  \begin{algorithmic}[1]
  \REQUIRE 
  Feature maps matrix: $\bm{F}$, 
  Jacobian for PCA-Grad-CAM: $\left\{\frac{\partial  \bm{p}}{\partial  \bm{x}^t} \mid t \in \mathbb{N}_{\le T}\right\}$, 
  Jacobian for SVM-Grad-CAM: $\left\{\frac{\partial  a(\bm{p})}{\partial  \bm{x}^t} \mid t \in \mathbb{N}_{\le T}\right\}$, 
  the number of principal components: $B$
  \ENSURE PCA-Grad-CAM: $\{\bm{P}_b^\oplus \mid b \in \mathbb{N}_{\le B}\}$, $\{\bm{P}_b^\ominus \mid b \in \mathbb{N}_{\le B}\}$, and SVM-Grad-CAM: $\bm{S}$
  \STATE \#\#\# Compute PCA-Grad-CAM \#\#\#
  \FOR{$b \leftarrow 1$ \TO $B$}
  \FOR{$t \leftarrow 1$ \TO $T$}
  \STATE $e^t_b \leftarrow \sum_{i=1}^{MN} \frac{\partial p_b}{\partial x_i^t}$, Eq.\ref{eq_pca_grad_cam}
  \ENDFOR
  \STATE $\bm{e}_b \leftarrow \begin{bmatrix}e_b^1 &\cdots& e_b^T\end{bmatrix}^\top$, Eq.\ref{eq_pca_grad_cam}
  \STATE $\bm{P}_b \leftarrow \bm{e}_b^\top \bm{F}$, Eq.\ref{eq_pca_grad_cam}
  \STATE $\bm{P}_b^\oplus \leftarrow  \mathrm{Relu}(\bm{P}_b), \ \bm{P}_b^\ominus \leftarrow \mathrm{Relu}(-\bm{P}_b)$, Eq.\ref{eq_pca_pm}
  \ENDFOR
  \STATE \#\#\# Compute SVM-Grad-CAM \#\#\#
  \FOR{$t \leftarrow 1$ \TO $T$}
  \STATE $s^t \leftarrow \sum_{i=1}^{MN} \frac{\partial a(\bm{p})}{\partial x_i^t}$, Eq.\ref{eq_svm_grad_cam}
  \ENDFOR
  \STATE $\bm{s} \leftarrow \begin{bmatrix}s^1 &\cdots& s^T\end{bmatrix}^\top$, Eq.\ref{eq_svm_grad_cam}
  \STATE $\bm{S} \leftarrow \mathrm{Relu} \left( \bm{s}^\top \bm{F} \right)$, Eq.\ref{eq_svm_grad_cam}
  \RETURN $\{\bm{P}_b^\oplus \mid b \in \mathbb{N}_{\le B}\}$, $\{\bm{P}_b^\ominus \mid b \in \mathbb{N}_{\le B}\}$, $\bm{S}$
   \end{algorithmic}
\end{algorithm}
\end{figure}

\section{Algorithm}
The procedures for computing PCA- and SVM- Grad-CAM are shown in Algorithms \ref{alg1} and \ref{alg2}.
Algorithm \ref{alg1} is the procedure for computing the Jacobians $\partial \bm{p}/\partial \bm{x}^t$ and $\partial a(\bm{p})/\partial \bm{x}^t$, which are required to compute the weights of the feature maps.
First, the feature maps $\{\bm{F}^1, \cdots, \bm{F}^T\}$ corresponding to the input image $X$ are extracted from the trained CNN $\Upsilon$ using the ``GetFeatureMap'' function.
The input vector $\bm{x}$ is obtained by flattening them.
Thereafter, the weights and biases, $\{(\bm{W}_1, \bm{\beta}_1), \cdots, (\bm{W}_l, \bm{\beta}_l)\}$, for the affine transformations are obtained using the ``GetParams'' function.
The procedures above correspond to lines 1--6 of Algorithm \ref{alg1}.

Next, procedures related to the PCA transformation are performed.
Specifically, the input vector $\bm{x}$ is compressed into a $B$-dimensional PCA feature vector $\bm{p}$ using the covariance matrix $\bm{\Sigma}$ and the eigenmatrix $\bm{V}$, as shown in lines 7--14 of Algorithm \ref{alg1}.
Subsequently, the Jacobians of the SVM layer (lines 15--20 of Algorithm \ref{alg1}), the PCA layer (lines 21--22 of Algorithm \ref{alg1}), and the dense layers (lines 23--31 of Algorithm \ref{alg1}) are computed.
From these, the closed-form Jacobians of $\{\frac{\partial  \bm{p}}{\partial  \bm{x}^t} \mid t \in \mathbb{N}_{\le T}\}$ and $\{\frac{\partial  a(\bm{p})}{\partial  \bm{x}^t} \mid t \in \mathbb{N}_{\le T}\}$ are computed using the gradient chain rule.
Finally, these Jacobians and the feature maps $\bm{F}$ are returned.
The procedure above is indicated in lines 32--43 of Algorithm \ref{alg1}.

Algorithm \ref{alg2} shows the exact procedures for computing PCA- and SVM- Grad-CAM using the Jacobians.
First, using $\{\frac{\partial  \bm{p}}{\partial  \bm{x}^t} \mid t \in \mathbb{N}_{\le T}\}$ returned from Algorithm \ref{alg1}, the PCA-Grad-CAMs $\bm{P}_b^\oplus$ and $\bm{P}_b^\ominus$ corresponding to the $b$-th principal component are computed, as shown in lines 1--9 of Algorithm \ref{alg2}.
Then, using $\{\frac{\partial  a(\bm{p})}{\partial  \bm{x}^t} \mid t \in \mathbb{N}_{\le T}\}$ returned from Algorithm \ref{alg1}, the SVM-Grad-CAM $\bm{S}$ is computed, as shown in lines 10--16 of Algorithm \ref{alg2}.

\section{Experiment 1: Visualization}\label{sec_vi}
\subsection{Outline}\label{sec_vi_a}
To verify the effectiveness of PCA- and SVM-Grad-CAMs, we conducted experiments using major datasets, such as MNIST~\cite{ref_mnist}, Fashion-MNIST~\cite{ref_fashion_mnist}, Blood-MNIST~\cite{ref_medmnist}, CIFAR-10~\cite{ref_cifar10}, Chest X-Ray~\cite{ref_covid_chest_original}, and Dog vs. Cat~\cite{ref_dog_cat}.
The Chest X-Ray dataset was created by Cohen et al.~\cite{ref_covid_chest_original} and has been used in several studies~\cite{ref_appl_covid1}.
The number of class labels and sizes of the training and test datasets are listed in Table~\ref{tab_1}.
The specific class label names are provided in Appendix \ref{app_b}.
For the CNN, we adopted the ImageNet-based VGG16~\cite{ref_vgg} and implemented four dense layers immediately after the flatten layer.
The dimensions of the first, second, and third dense layers were 40, 30, and 20, respectively. 
Because the fourth dense layer was used for classification, its dimension was the same as the number of class labels.
The size of the input images was $200 \times 200 \times 3$ for all datasets.
When adopting this input size with VGG16~\cite{ref_vgg}, the Grad-CAM output size is $6 \times 6$.
Minibatch learning was performed with a batch size of 32.
During the learning process, all convolutional layers were frozen, and only the dense layers were trained.
An NVIDIA RTX A6000 (48 GB) was used in this experiment.

\begin{figure}[pos=tp]
    \begin{center}
        \includegraphics[scale=0.8]{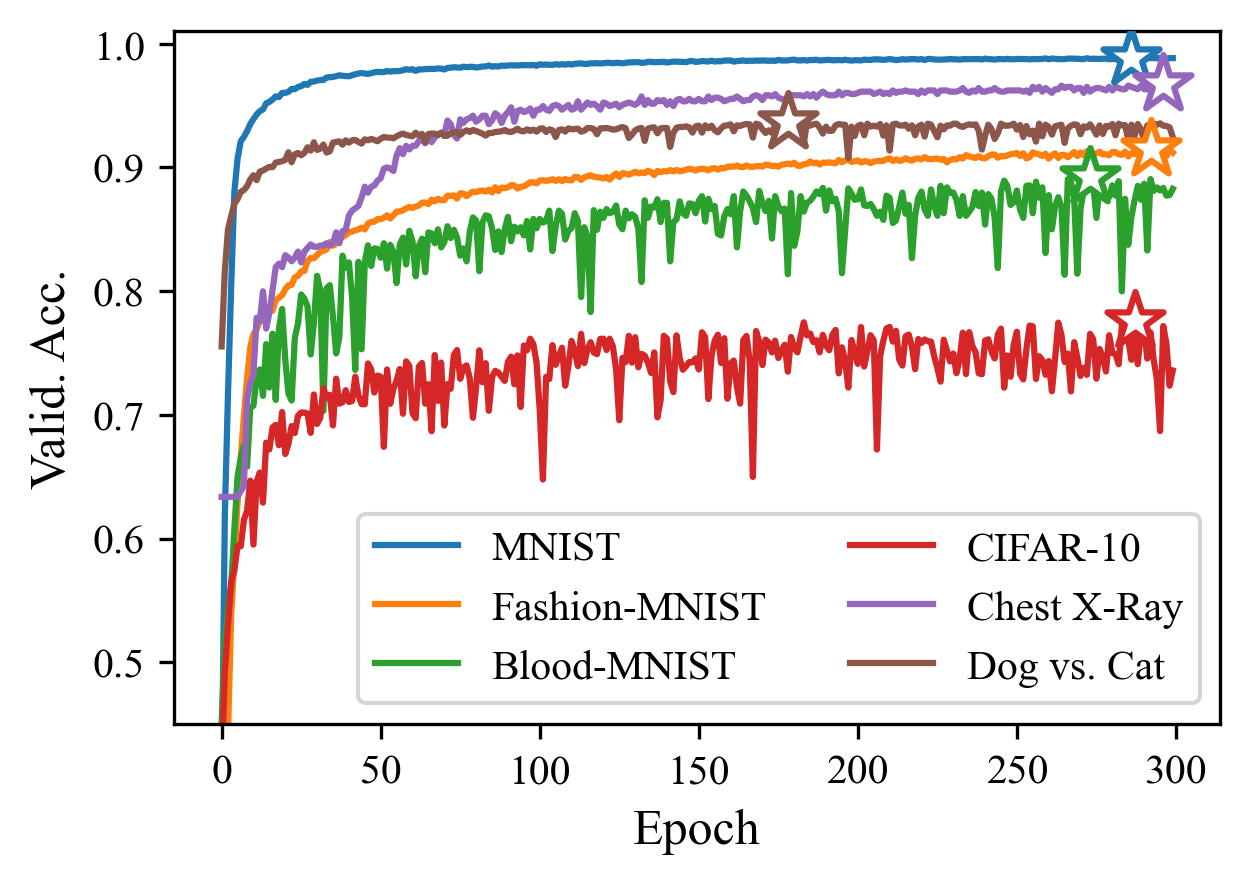}
        \caption{Relationship between epoch and validation accuracies. A star marker $\largestar$ represents an optimal epoch.}
        \label{fig_valacc}
    \end{center}
\end{figure}

Stochastic Gradient Descent (SGD)~\cite{ref_alexnet} was used as the optimizer, cross-entropy was used as the loss function, and the learning rate was set to $10^{-3}$.
We allocated 20\% of the training data as validation data and searched for the optimal number of epochs that yielded the highest validation accuracy.
The number of epochs was searched in the range of 1--300.
The PCA layer was applied to the second dense layer ($l=2$), and the dimension of the PCA feature vector was set to $B=3$.
In general, the dimensions are determined based on the contribution ratio computed from the eigenvalues. However, because the objective of this study was to evaluate PCA-Grad-CAM, the dimension of the PCA feature vector was set to the same value for all datasets.
With a three-dimensional PCA feature vector, the feature space can be visually inspected.
The PCA transformation matrix was constructed using only the training dataset; the test dataset was not used.
As the SVM inserted after the PCA layer, we adopted two types: Linear SVM (Lin-SVM) and RBF-SVM. 
The hyperparameters of these SVMs were set as follows: the cost parameter was 1, and $\gamma$ was 1.
Because our objective was to verify SVM-Grad-CAM, no hyperparameter tuning was performed for the SVMs.
For these models, PCA-Grad-CAM and SVM-Grad-CAM were implemented.

\begin{table}[pos=tp]
\tabcolsep = 1.5pt
  \caption{Dataset informations and estimation performances.}
  \label{tab_1}
  \resizebox{\linewidth}{!}{
  \centering
  \begin{tabular}{lcccc}
\toprule
&Classes /&\multicolumn{3}{c}{Macro F1 score (training/test)} \\
\multicolumn{1}{c}{Dataset} & $N_\mathrm{train}$ / $N_\mathrm{test}$ &CNN&Lin-SVM&RBF-SVM \\ 
\midrule
MNIST \cite{ref_mnist} &10/60,000/10,000& .993/.991 &.986/.980& .997/.989\\
F-MNIST \cite{ref_fashion_mnist} &10/60,000/10,000& .925/.903& .922/.897&.988/.943\\
B-MNIST \cite{ref_medmnist} &8/11,959/3,421& .944/.877& .925/.873&.984/.909\\
CIFAR-10 \cite{ref_cifar10} &10/50,000/10,000&.899/.767& .789/.735&.892/.794\\
Chest X-Ray \cite{ref_covid_chest_original} &3/5,144/1,288& .966/.956& .998/.975&.998/.975\\
Dog vs. Cat \cite{ref_dog_cat} &2/18,610/4,652& .967/.931& .995/.930&.995/.931\\
\bottomrule
    \multicolumn{5}{r}{$N_\mathrm{train}$: training samples, $N_\mathrm{test}$: test samples, F: Fashion, B: Blood}
  \end{tabular}
  }
\end{table}

\subsection{Result and discussion}
\subsubsection{Model development}
The relationship between epochs and validation accuracy is shown in Figure \ref{fig_valacc}.  
The figure indicates that the validation accuracies have converged.
The performances obtained by adopting the epoch that maximized the validation accuracy are listed in Table \ref{tab_1}.  
From this table, it can be observed that the models exhibit good generalization performance.  
Therefore, we decided to use these CNNs in our experiments.

The contribution ratios of the three principal components used in PCA are shown in Figure \ref{fig_eigen}.  
These values are normalized eigenvalues calculated as $\lambda_b^\prime = 100 \times \lambda_b/\sum_b \lambda_b$.
Note that $\lambda_b$ denotes the $b$-th eigenvalue.
Figure \ref{fig_eigen} shows that, for the models trained on the MNIST, Fashion-MNIST, Blood-MNIST, and CIFAR-10 datasets, the first to third principal components had relatively high contribution ratios. 
By contrast, for models trained on the Chest X-Ray and Dog vs. Cat datasets, only the first principal component had a high contribution ratio.
The three-dimensional PCA feature spaces of all the models are shown in Figure \ref{fig_s_train_test}.
The figure indicates that the scatter plots for each class are somewhat separated.
However, in the Dog vs. Cat model, PCA features such as $p_2$ and $p_3$ with low contribution ratios did not seem to improve the classification performance.
Therefore, we concluded that PCA features with large contribution ratios are important.
Note that the total contribution ratios of the 1--3 principal components were greater than 75\% for all the models.

Previous studies have shown that incorporating a PCA layer into a CNN can improve its classification performance~\cite{ref_pca_svm1, ref_pca_svm2, ref_pca_svm3, ref_pca_svm4, ref_pca_svm5, ref_pca_svm6}.
Figure \ref{fig_s_train_test} also demonstrates the effectiveness of PCA.
However, we could not determine the focus regions of each principal component.
PCA-Grad-CAM reveals this, and the results of its application are described in the following section.

\begin{figure}[pos=tp]
    \begin{center}
        \includegraphics[scale=0.85]{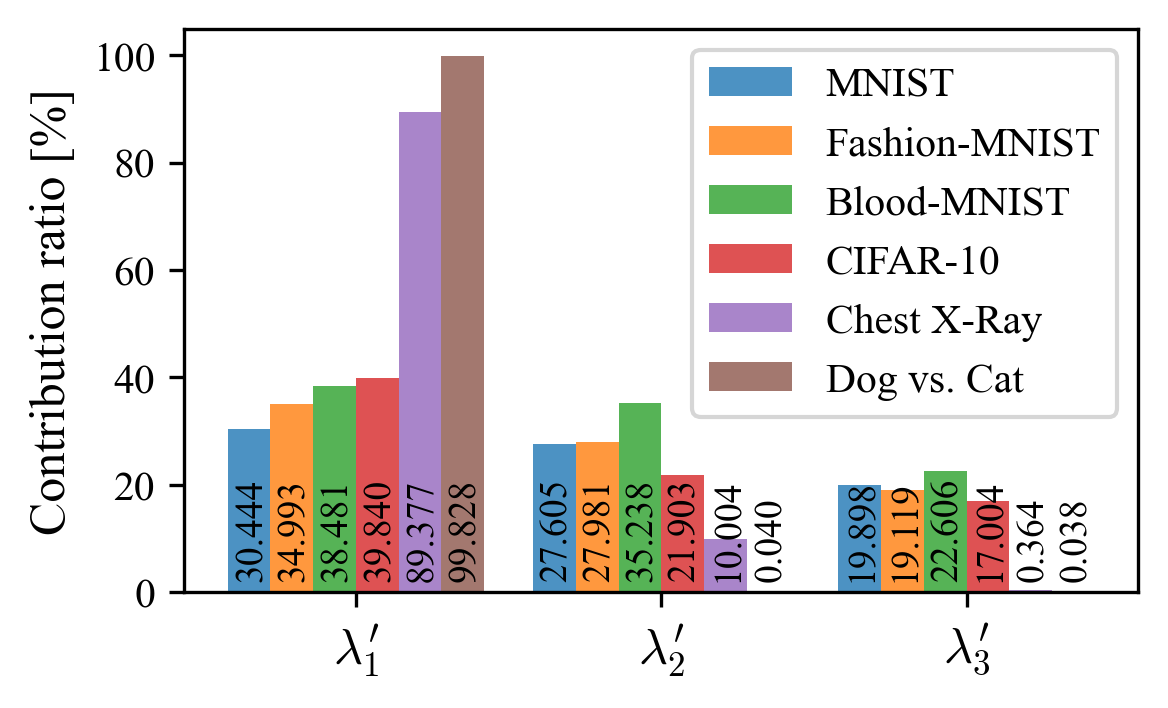}
        \caption{
        Contribution ratios of the eigenvectors $\bm{v}_1$, $\bm{v}_2$, and $\bm{v}_3$.
Here, $\lambda_b^\prime$ represents the normalized eigenvalues, defined as $\lambda_b^\prime = 100 \times \lambda_b / \sum_b \lambda_b$.
        }
        \label{fig_eigen}
    \end{center}
\end{figure}

\begin{figure*}[pos=tp]
    \centering
    \includegraphics[scale=0.65]{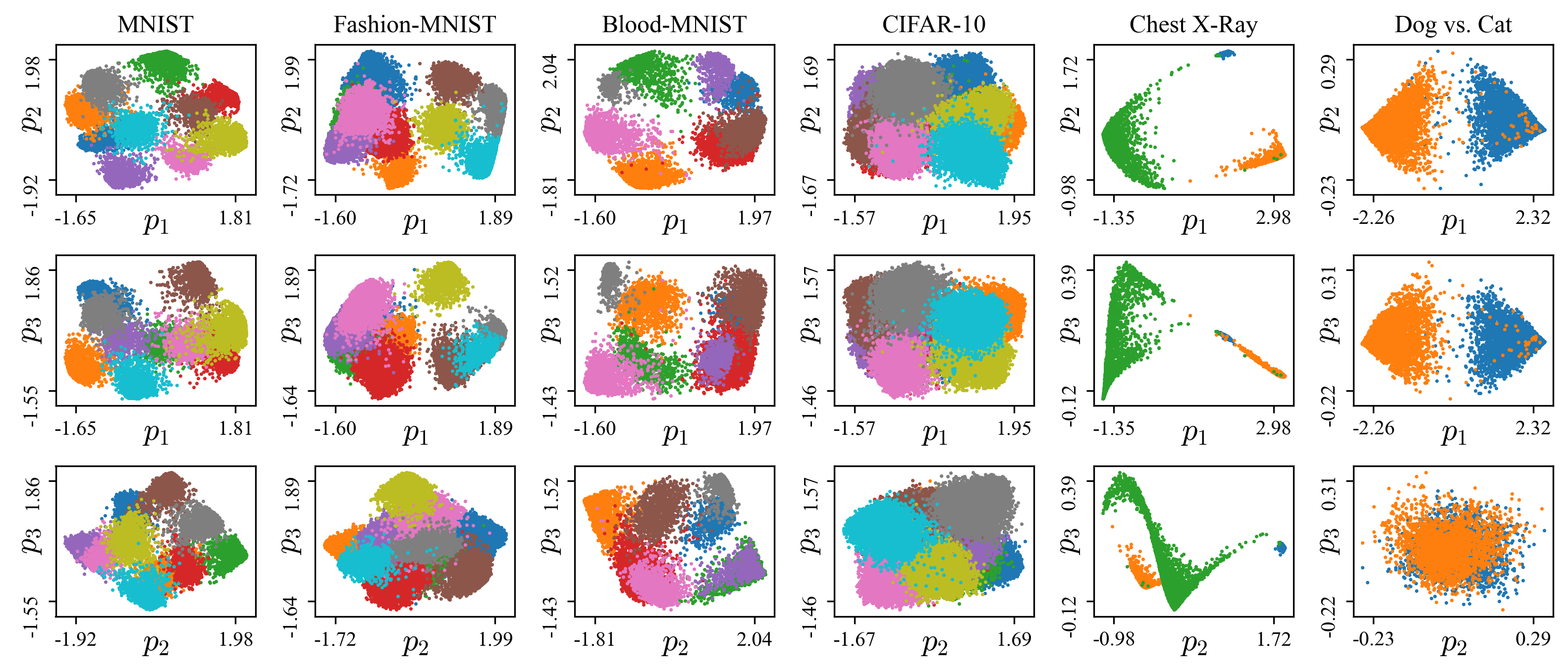}
    \vspace{2mm}
    \includegraphics[scale=0.65]{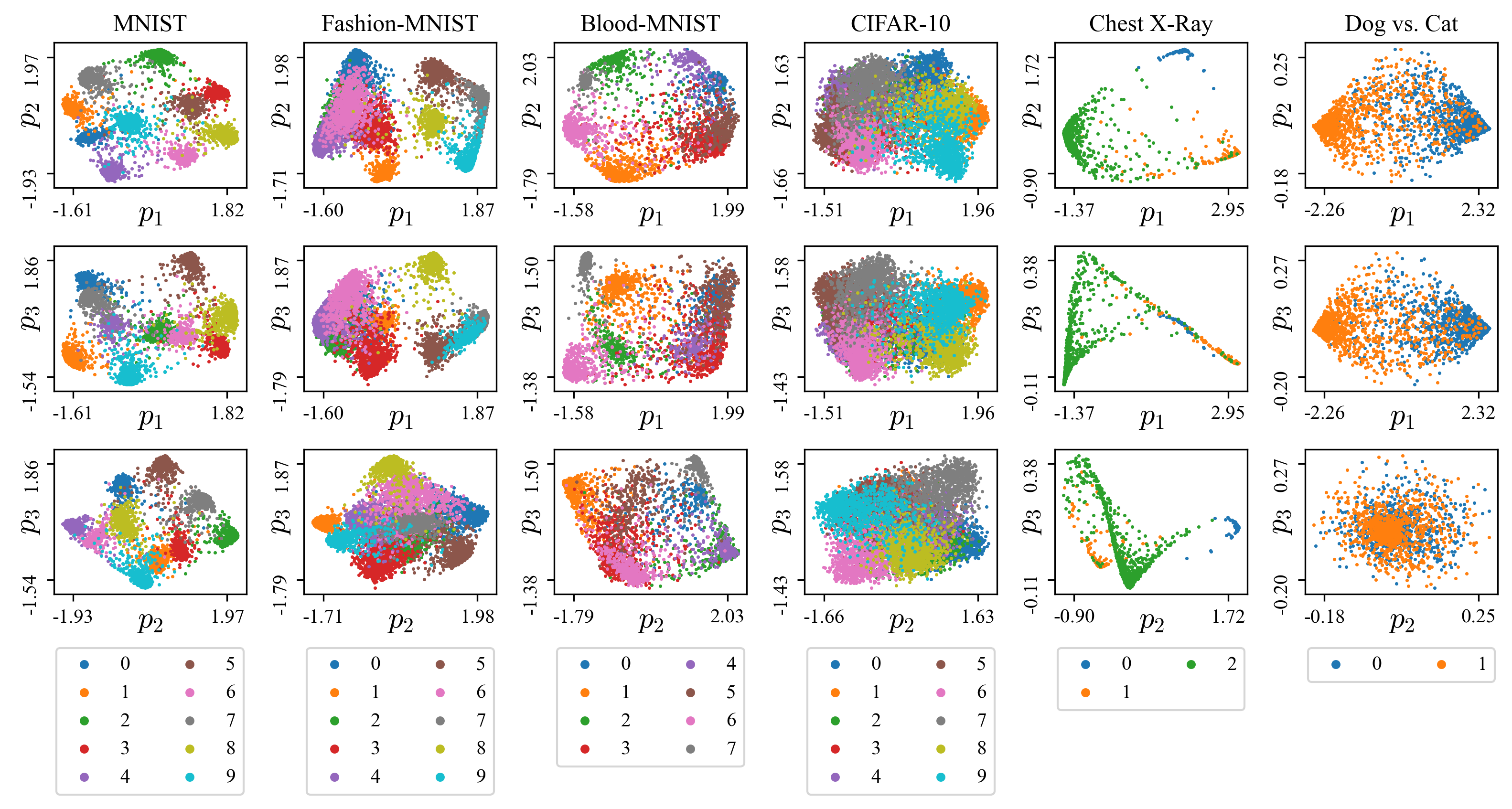}
    \caption{
    Three-dimensional PCA feature spaces comprising $\bm{p} = \begin{bmatrix}p_1 \ p_2 \ p_3\end{bmatrix}^\top$. 
    The meanings of the labels from 0 to 9 are described in the Appendix \ref{app_b}.
    The top and bottom figures are scatter plots corresponding to the training and test data, respectively.
    }
    \label{fig_s_train_test}
\end{figure*}

\begin{figure*}[pos=tp]
    \begin{center}
        \includegraphics[scale=0.65]{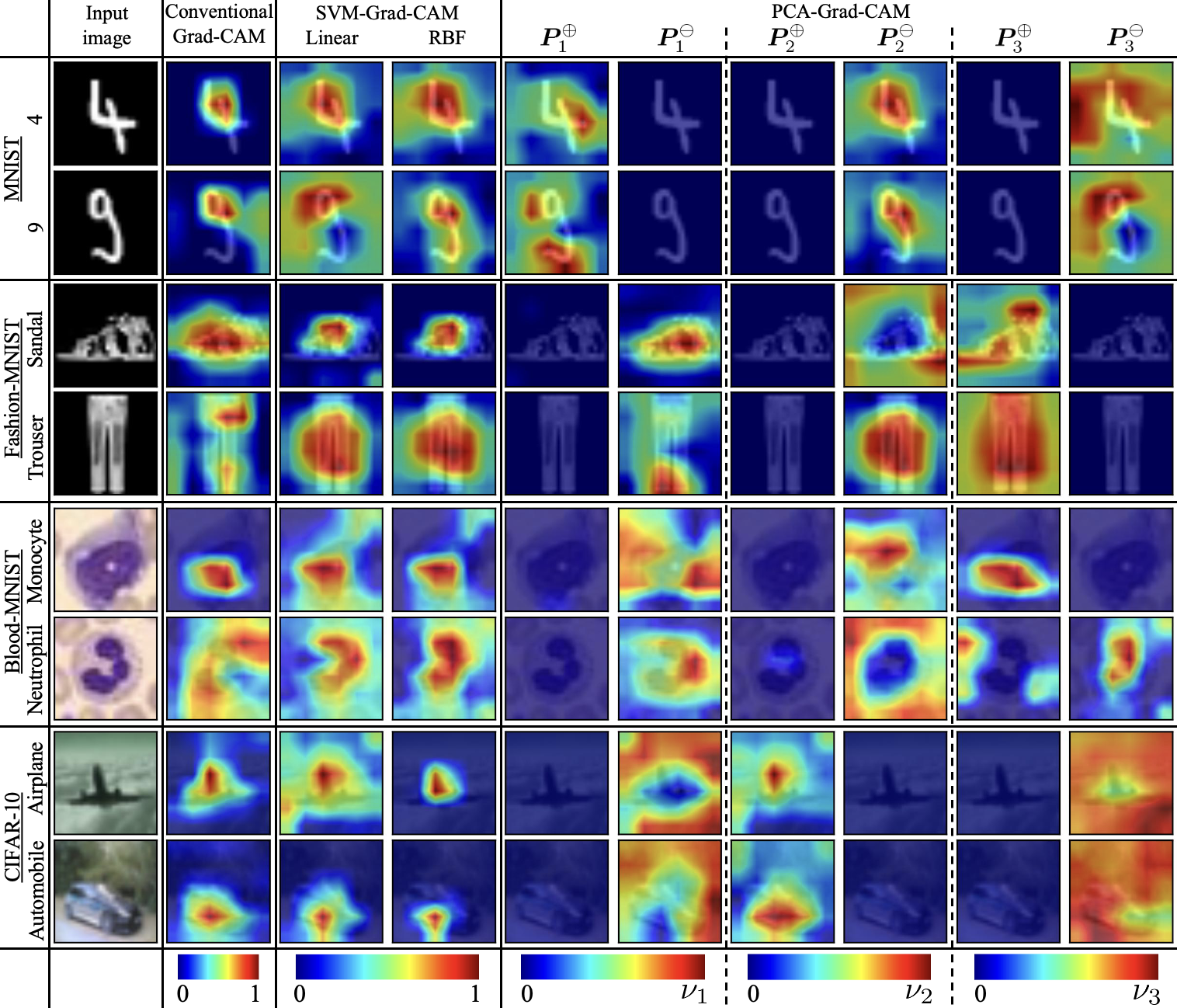}
        \caption{
        Conventional-Grad-CAM, PCA-Grad-CAM, and SVM-Grad-CAM (MNIST~\cite{ref_mnist}, Fashion-MNIST~\cite{ref_fashion_mnist}, Blood-MNIST~\cite{ref_medmnist}, and CIFAR-10~\cite{ref_cifar10}).}
        \label{fig_all_cams}
    \end{center}
\end{figure*}

\begin{figure}[pos=tp]
    \begin{center}
        \includegraphics[scale=0.5]{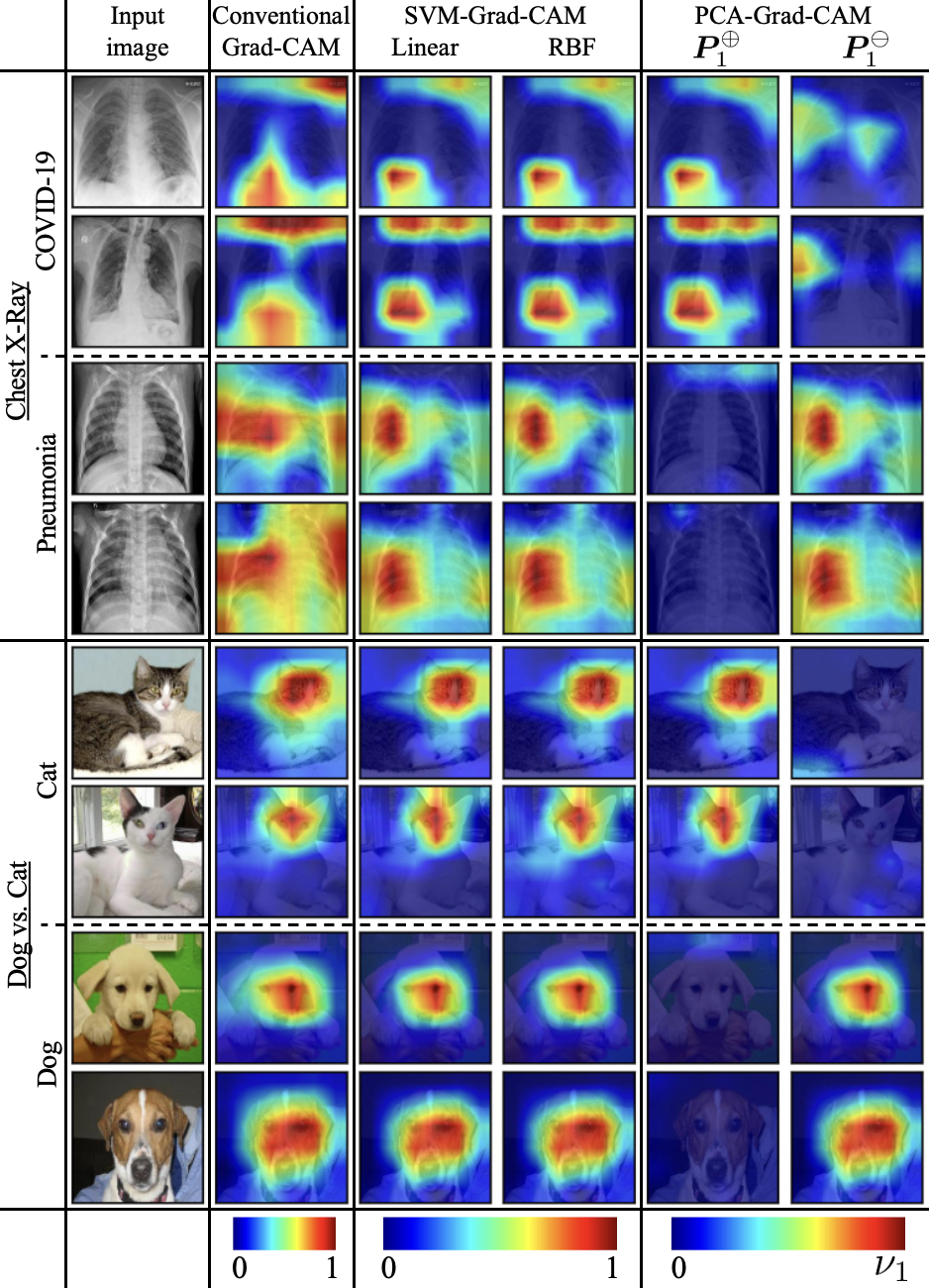}
        \caption{
        Conventional-Grad-CAM, PCA-Grad-CAM, and SVM-Grad-CAM (Chest X-Ray~\cite{ref_covid_chest_original} and Dog vs. Cat~\cite{ref_dog_cat}).}
        \label{fig_all_cams_eigen1}
    \end{center}
\end{figure}

\subsubsection{Visualization}
Next, the proposed methods, PCA-Grad-CAM and SVM-Grad-CAM, were computed using Algorithms \ref{alg1} and \ref{alg2}.
For comparison, the existing method, conventional Grad-CAM~\cite{grad_cam_origin}, was also used.
Figure \ref{fig_all_cams} shows the cases in which the contribution ratios of the 1--3 principal components were relatively high, namely, the CNNs trained on the MNIST, Fashion-MNIST, Blood-MNIST, and CIFAR-10 datasets.
In numerous cases, the focus regions of the SVM-Grad-CAM are similar to those of the conventional Grad-CAM.
In contrast, for the ``Trouser'' class of Fashion-MNIST and the ``Neutrophil'' class of Blood-MNIST, the conventional Grad-CAM and the SVM-Grad-CAM produce different results.
The conventional Grad-CAM was computed using a standard CNN, whereas the SVM-Grad-CAM was computed using an SVM embedded within the CNN.
Therefore, it is natural that some results would be similar, whereas others would differ.

Next, we examined the results of the PCA-Grad-CAM.
$\bm{P}_b^\oplus$ and $\bm{P}_b^\ominus$ represent the Grad-CAMs for the $b$-th principal component $p_b$.
In general, we observed that the focus regions differed across the principal components.
For instance, for the class ``4'' in MNIST, $\bm{P}_1^\oplus$ focuses on the right side, $\bm{P}_2^\ominus$ on the top, and $\bm{P}_3^\ominus$ on both the top and the background.
For the class ``9'', $\bm{P}_1^\oplus$ focuses on the left and bottom, $\bm{P}_2^\ominus$ on the center, and $\bm{P}_3^\ominus$ on the top and background.

For the CIFAR-10 dataset, $\bm{P}_1^\ominus$ focuses on the background, whereas $\bm{P}_2^\oplus$ focuses on the airplane and car.
$\bm{P}_3^\ominus$ contains almost all the activated pixels.
This indicates that the gradients of $p_3$ with respect to many feature maps are negative.
Similarly, in the cases of Fashion-MNIST and Blood-MNIST, the focus regions differed among the principal components.
This indicates that the 1--3 principal components extracted features from different regions.
Note that cases exist in which PCA-Grad-CAM and SVM-Grad-CAM are similar.
This is probably because the SVM performs classification using only the principal component.
For instance, in the MNIST class ``4'', the SVM-Grad-CAM is extremely similar to $\bm{P}_2^\ominus$.
This suggests that the second dimension of the PCA feature vector, $p_2$, likely contributes to classification.

Next, in Figure \ref{fig_all_cams_eigen1}, we show the results for cases in which only the first principal component has a high contribution ratio, namely, the Chest X-Ray~\cite{ref_covid_chest_original} and Dog vs. Cat~\cite{ref_dog_cat} models.
As shown in Figure \ref{fig_eigen}, because only the first principal component has a high contribution ratio, only $\bm{P}_1^\oplus$ and $\bm{P}_1^\ominus$ were computed.
Figure \ref{fig_all_cams_eigen1} indicates that PCA-Grad-CAM and SVM-Grad-CAM are highly similar.
This likely indicates that the SVM performs classification using only the first principal component, $p_1$.
For the Chest X-Ray model, $\bm{P}_1^\oplus$ corresponds to ``COVID-19'' class, and $\bm{P}_1^\ominus$ corresponds to ``Pneumonia'' class.
It can be considered that regions associated with COVID-19 locally increased the first principal component, $p_1$, whereas regions associated with pneumonia locally decreased $p_1$.
Similarly, for the Dog vs. Cat model, $\bm{P}_1^\oplus$ and $\bm{P}_1^\ominus$ correspond to the faces of cats and dogs, respectively.
It can be considered that the face of the cat locally increased the first principal component, $p_1$, whereas the face of the dog locally decreased $p_1$.

Previous studies have shown that the conventional Grad-CAM~\cite{grad_cam_origin} can visualize the focus regions of a CNN.
However, when PCA or SVM is added to a CNN, the focus regions cannot be directly visualized.
The proposed PCA-Grad-CAM and SVM-Grad-CAM addressed this problem; thus, we believe that this study has value.

\section{Experiment 2: Quantitative reliability analysis}
\subsection{Outline} \label{subsec_71}
We evaluated whether PCA-Grad-CAM and SVM-Grad-CAM successfully identified and visualized the regions that contributed to the PCA and SVM scores, respectively.
Therefore, we adopted the insertion and deletion metrics~\cite{ref_add_01} to evaluate the proposed methods.
They have been widely used in many studies and are regarded as reliable evaluation metrics~\cite{ref_lift_cam_ori, ref_add_02}.
The insertion metric is defined as the area under the curve (AUC) of the output score variation when the image is progressively restored in descending order of heatmap values.
In contrast, the deletion metric is defined as the AUC of the output score as the image is progressively perturbed in descending order of the heatmap values.
For SVM-Grad-CAM $\bm{S}$ and the positive map $\bm{P}_b^\oplus$ of PCA-Grad-CAM, a larger insertion metric indicates better visualization performance, whereas a smaller deletion metric indicates better visualization performance.
In contrast, for the negative map $\bm{P}_b^\ominus$ of PCA-Grad-CAM, a smaller insertion metric indicates better visualization performance, whereas a larger deletion metric indicates better visualization performance.
This is because the negative map $\bm{P}_b^\ominus$ visualizes the contribution to the decrease in the $b$-th principal component $p_b$.
Because the interpretation of the scores differs only for $\bm{P}_b^\ominus$, this could potentially cause confusion. Therefore, we adopted $1 - \mathrm{AUC}$ as the evaluation metric exclusively for $\bm{P}_b^\ominus$.
Consequently, for all visualization methods, better visualizations correspond to higher insertion metrics and lower deletion metrics.

The evaluation targets were SVM-Grad-CAMs (Linear and RBF) and PCA-Grad-CAM for the first principal component.
As shown in Figure~\ref{fig_all_cams}, the negative map $\bm{P}_1^\ominus$ for MNIST and the positive map $\bm{P}_1^\oplus$ for Fashion-MNIST, Blood-MNIST, and CIFAR-10 were nearly zero. Consequently, almost no meaningful heatmaps were obtained in these cases.
This tendency was also observed for images other than those shown in Figure~\ref{fig_all_cams}.
Therefore, for MNIST, only $\bm{P}_1^\oplus$ was evaluated, whereas for Fashion-MNIST, Blood-MNIST, and CIFAR-10, only $\bm{P}_1^\ominus$ was evaluated.
The insertion and deletion metrics were applied to the test datasets rather than the training datasets.

In addition, to assess whether the regions associated with the PCA and SVM scores were appropriately visualized, we also computed random insertion and deletion metrics as baselines.
The insertion and deletion metrics are in one-to-one correspondence with the random insertion and deletion baselines for each image.
Therefore, we performed the Wilcoxon signed-rank test~\cite{ref_add_wilcox}, a non-parametric statistical test used to compare two related samples.
If the AUC is significantly better than that of the random baseline, PCA-Grad-CAM and SVM-Grad-CAM are considered to appropriately visualize the regions associated with the PCA and SVM scores, respectively.
Note that, because the statistical test is repeated multiple times, there is a possibility of observing significant differences by chance.
To address this issue, the $p$-values were adjusted using the Holm correction~\cite{ref_add_holm}.
The combination of the Wilcoxon signed-rank test and the Holm correction is a standard approach commonly used to evaluate machine learning performance~\cite{ref_add_wh1, ref_add_wh2, ref_add_wh3}.

\subsection{Result and discussion}
The medians of the insertion AUC and the random insertion AUC, computed from the test datasets, are shown in Table~\ref{tab_4}.
The table indicates that, for all items, the proposed methods outperform the random baseline in terms of AUC.
This indicates that, by progressively reconstructing the images based on the regions identified as important by PCA-Grad-CAM and SVM-Grad-CAM, the PCA and SVM scores increase rapidly.
The medians of the deletion AUC and the random deletion AUC, computed from the test datasets, are shown in Table~\ref{tab_5}, indicating that the proposed methods generally produce lower AUCs than the random baseline.
This indicates that, by progressively removing the regions identified as important by PCA-Grad-CAM and SVM-Grad-CAM from the images, the PCA and SVM scores decrease sharply.
Based on these observations, PCA-Grad-CAM and SVM-Grad-CAM can be considered to appropriately visualize regions associated with PCA and SVM, respectively.

For PCA-Grad-CAM on CIFAR-10 and Chest X-Ray, a significant difference was observed for the insertion AUC, whereas no significant difference was found for the deletion AUC.
In some datasets, PCA-Grad-CAM highlights regions that sufficiently reflect the PCA scores, although these regions are probably not the only ones necessary.
This point suggests that further improvements to the proposed methods may be necessary in future work.

\begin{table}[pos=tp]
\tabcolsep = 2pt
\caption{Medians of insertion AUC (Ins.) and random insertion AUC (Rand. Ins.) by test datasets. 
{\bfseries Bold values} indicate statistically superior performance.
{\bfseries Higher} is better. 
For the meaning of ``--'', see Subsection~\ref{subsec_71}.}
  \label{tab_4}
  \resizebox{\linewidth}{!}{
  \centering
  \begin{tabular}{llccc}
\toprule
\multicolumn{1}{c}{Dataset} & \multicolumn{1}{c}{Method} & Ins. & Rand. Ins. & Adj. $p$\\
\midrule
MNIST             & SVM-Grad-CAM (Linear) &{\bfseries .676}&.496 & **\\
\cite{ref_mnist} & SVM-Grad-CAM (RBF) &{\bfseries .518} &.320 & **\\ 
                        & PCA-Grad-CAM $\oplus$ &{\bfseries .754}&.441& ** \\ 
                        & PCA-Grad-CAM $\ominus$ & -- &--& -- \\ \midrule
Fashion- & SVM-Grad-CAM (Linear) &{\bfseries .668}&.546& **\\
MNIST & SVM-Grad-CAM (RBF) &{\bfseries .667}&.588& **\\
\cite{ref_fashion_mnist} & PCA-Grad-CAM $\oplus$ &--&--& -- \\ 
                        & PCA-Grad-CAM $\ominus$ &{\bfseries .669}&.607& ** \\ \midrule
Blood- & SVM-Grad-CAM (Linear) &{\bfseries .635}&.488& **\\
MNIST & SVM-Grad-CAM (RBF) &{\bfseries .591}&.432& **\\
\cite{ref_medmnist} & PCA-Grad-CAM $\oplus$ &--&--& --\\ 
& PCA-Grad-CAM $\ominus$ &{\bfseries .669}&.534& **\\  \midrule
CIFAR-10 & SVM-Grad-CAM (Linear) &{\bfseries .529}&.451& **\\
\cite{ref_cifar10} & SVM-Grad-CAM (RBF) &{\bfseries .537}&.480& **\\
 & PCA-Grad-CAM $\oplus$ &-- &--&--\\ 
& PCA-Grad-CAM $\ominus$ &{\bfseries .530}&.500& **\\  \midrule
 Chest & SVM-Grad-CAM (Linear) &{\bfseries .904}&.834& **\\
X-Ray & SVM-Grad-CAM (RBF) &{\bfseries .278}&.234& **\\
\cite{ref_covid_chest_original} & PCA-Grad-CAM $\oplus$ & {\bfseries .580}&.564& **\\ 
& PCA-Grad-CAM $\ominus$ &{\bfseries .597}&.468& **\\  \midrule
 Dog vs. & SVM-Grad-CAM (Linear) &{\bfseries .821}&.588& **\\
Cat~\cite{ref_dog_cat} & SVM-Grad-CAM (RBF) &{\bfseries .678}&.563& **\\
 & PCA-Grad-CAM $\oplus$ &{\bfseries .837}&.601& **\\ 
& PCA-Grad-CAM $\ominus$ &{\bfseries .808}&.581& **\\
\bottomrule
    \multicolumn{5}{r}{Adj. $p$: Wilcoxon signed-rank test with Holm correction.}\\
    \multicolumn{5}{r}{*: $p < .05$, **: $p < .01$}
  \end{tabular}
  }
\end{table}

\begin{table}[pos=tp]
\tabcolsep = 2pt
\caption{Medians of deletion AUC (Del.) and random deletion AUC (Rand. Del.) by test datasets. 
{\bfseries Bold values} indicate statistically superior performance.
{\bfseries Lower} is better.
For the meaning of ``--'', see Subsection~\ref{subsec_71}.
}
  \label{tab_5}
  \resizebox{\linewidth}{!}{
  \centering
  \begin{tabular}{llccc}
\toprule
\multicolumn{1}{c}{Dataset} & \multicolumn{1}{c}{Method} & Del. & Rand. Del. & Adj. $p$\\
\midrule
MNIST             & SVM-Grad-CAM (Linear) &{\bfseries .416} & .532 & **\\
\cite{ref_mnist} & SVM-Grad-CAM (RBF) &{\bfseries .316} &.353 & ** \\ 
                        & PCA-Grad-CAM $\oplus$ &{\bfseries .299} &.488& ** \\ 
                        & PCA-Grad-CAM $\ominus$ & -- &--& -- \\ \midrule
Fashion- & SVM-Grad-CAM (Linear) &{\bfseries .525}&.555& **\\
MNIST & SVM-Grad-CAM (RBF) &{\bfseries .558}&.599& **\\
\cite{ref_fashion_mnist} & PCA-Grad-CAM $\oplus$ & -- & -- & -- \\ 
                        & PCA-Grad-CAM $\ominus$ &{\bfseries .601}&.609& * \\ \midrule
Blood- & SVM-Grad-CAM (Linear) &{\bfseries .450}&.504& **\\
MNIST& SVM-Grad-CAM (RBF) &{\bfseries .437}&.464& **\\
\cite{ref_medmnist} & PCA-Grad-CAM $\oplus$ & -- & -- & --\\ 
& PCA-Grad-CAM $\ominus$ &{\bfseries .484}&.554& **\\  \midrule
CIFAR-10 & SVM-Grad-CAM (Linear) &{\bfseries .427}&.458& **\\
\cite{ref_cifar10} & SVM-Grad-CAM (RBF) &{\bfseries .463}&.492& **\\
 & PCA-Grad-CAM $\oplus$ &-- &--&--\\ 
 & PCA-Grad-CAM $\ominus$ &.507&.496& n.s.\\  \midrule
 Chest & SVM-Grad-CAM (Linear) & {\bfseries .714} & .892 & ** \\
X-Ray & SVM-Grad-CAM (RBF) & {\bfseries .215} &.236& ** \\
\cite{ref_covid_chest_original} & PCA-Grad-CAM $\oplus$ & .584 &.572& n.s.\\ 
& PCA-Grad-CAM $\ominus$ &.467&.497& n.s. \\  \midrule
 Dog vs. & SVM-Grad-CAM (Linear) &{\bfseries .470}& .622& **\\
Cat \cite{ref_dog_cat} & SVM-Grad-CAM (RBF) &{\bfseries .476}&.586& **\\
 & PCA-Grad-CAM $\oplus$ &{\bfseries .510}&.643& **\\ 
& PCA-Grad-CAM $\ominus$ &{\bfseries .439}&.607& **\\
\bottomrule
    \multicolumn{5}{r}{Adj. $p$: Wilcoxon signed-rank test with Holm correction.}\\
    \multicolumn{5}{r}{*: $p < .05$, **: $p < .01$}
  \end{tabular}
  }
\end{table}

\section{Conclusion}
In this study, we proposed a method for computing Grad-CAM for PCA and SVM layers embedded in a CNN. 
This allows us to identify the focus regions of PCA and SVM, which could not be observed using conventional methods (see Figures \ref{fig_all_cams} and \ref{fig_all_cams_eigen1}).
In addition, using the gradient chain rule, we derived the closed-form expressions of the Jacobians required for PCA- and SVM-Grad-CAM (see Table \ref{tab_full}).

One may argue that the same gradients can be obtained by automatic differentiation if the entire pipeline, including the PCA and SVM, is re-implemented in a deep learning framework.
Indeed, the closed-form Jacobians in Table~\ref{tab_full} are mathematically equivalent to the gradients computed by automatic differentiation.
However, the contribution of this study does not lie in merely computing gradients.
First, we defined what should be visualized for PCA and SVM layers, such as the separation of positive and negative effects for each principal component, which is not determined by automatic differentiation itself.
Second, the closed-form expressions reveal how the visualization depends on the model components, which cannot be read from gradient values alone.
For instance, Equations~\ref{eq_svm_abstk} and \ref{eq_29} show that, with the linear kernel, $\partial a(\bm{p})/\partial \bm{p}$ is a constant determined only by the trained SVM parameters and does not depend on the PCA feature vector $\bm{p}$.
Because the Jacobian of the PCA layer, $\bm{V}^\top$, is also constant, the variation of SVM-Grad-CAM across input images arises solely from the CNN part in this case.
In contrast, Theorems~\ref{th_0} and \ref{th_1} show that, with the RBF kernel, the hyperparameter $\gamma$ and the distance between each support vector and the PCA feature vector determine the contribution of that support vector to SVM-Grad-CAM.
Third, because the expressions are explicit, the proposed methods can be implemented directly from the parameters of PCA and SVM models trained with general machine learning libraries (e.g., scikit-learn), without re-implementing the PCA and SVM in a differentiable framework.

Note that numerous other dimensionality reduction and classification methods applicable to CNN features in addition to PCA and SVM exist.
In the future, we aim to extend Grad-CAM to these methods as well.
Although the proposed method is an extension of Grad-CAM, the Jacobians derived in this study (see Table \ref{tab_full}) can be directly used to visualize PCA and SVM layers in other first-order derivative-based visualization methods.
Specifically, by slightly modifying the formulations of XGrad-CAM~\cite{ref_xgrad_origin} and HiRes-CAM~\cite{ref_hirescom_origin}, these methods can also be applied to PCA and SVM layers.
This will be investigated in future research.
However, for Grad-CAM++~\cite{ref_gradcam_pp_origin}, which uses higher-order derivatives, and LIFT-CAM~\cite{ref_lift_cam_ori}, which does not rely on derivatives, the PCA and SVM layers cannot be visualized using only the Jacobians.
Extensions that allow visualization of the PCA and SVM layers for such methods are also needed.
We plan to address this issue in our future studies.

\appendix
\section{Mathematical representation} \label{app_a}
In this study, a vector such as $\bm{\zeta} \in \mathbb{R}^Z$ is assumed to be a column vector.
Then, the $z$-th component of $\bm{\zeta}$ is denoted by $(\bm{\zeta})_z$, that is, 
\begin{align}
\bm{\zeta} = \begin{bmatrix} \zeta_1 & \cdots & \zeta_Z\end{bmatrix}^\top \in \mathbb{R}^Z
\Rightarrow (\bm{\zeta})_z = \zeta_z, \forall z \in \mathbb{N}_{\le Z}. \nonumber
\end{align}
Moreover, the gradient of the output vector $\bm{\eta} = \begin{bmatrix}\eta_1 &\cdots& \eta_H \end{bmatrix}^\top \in \mathbb{R}^H$ with respect to the input vector $\bm{\zeta} = \begin{bmatrix}\zeta_1 &\cdots &\zeta_Z \end{bmatrix}^\top \in \mathbb{R}^Z$ is defined as
\begin{align}
\frac{\partial \bm{\eta}}{\partial \bm{\zeta}} = \begin{bmatrix}
\frac{\partial \eta_1}{\partial \zeta_1} & \cdots &\frac{\partial \eta_1}{\partial \zeta_Z} \\
\vdots & \ddots & \vdots\\
\frac{\partial \eta_H}{\partial \zeta_1} & \cdots &\frac{\partial \eta_H}{\partial \zeta_Z} \\
\end{bmatrix} \in \mathbb{R}^{H \times Z}. \nonumber
\end{align}
In this study, we refer to this as the ``Jacobian''.
We also define the scalar derivative of a vector and the vector derivative of a scalar as
\begin{align}
\frac{\partial \bm{\eta}}{\partial \zeta_z} &= \begin{bmatrix}
\frac{\partial \eta_1}{\partial \zeta_z}&
\cdots &
\frac{\partial \eta_H}{\partial \zeta_z} 
\end{bmatrix}^\top \in \mathbb{R}^{H}, \nonumber \\
\frac{\partial \eta_h}{\partial \bm{\zeta}} &= \begin{bmatrix}
\frac{\partial \eta_h}{\partial \zeta_1} & \cdots &\frac{\partial \eta_h}{\partial \zeta_Z}
\end{bmatrix} \in \mathbb{R}^{1 \times Z}, \nonumber
\end{align}
respectively. 
By adopting this representation, the Jacobian can be expressed as
\begin{align}
\frac{\partial \bm{\eta}}{\partial \bm{\zeta}} = 
\begin{bmatrix}
\frac{\partial \bm{\eta}}{\partial \zeta_1} & \cdots  & \frac{\partial \bm{\eta}}{\partial \zeta_Z}
\end{bmatrix}=
\begin{bmatrix}
\frac{\partial \eta_1}{\partial \bm{\zeta}} \\ \vdots  \\ \frac{\partial \eta_H}{\partial \bm{\zeta}}
\end{bmatrix} \in \mathbb{R}^{H \times Z}. \nonumber
\end{align}

\section{Class labels}\label{app_b}
\begin{itemize}

\item
Fashion-MNIST~\cite{ref_fashion_mnist}
0: T-shirt/top, 1: Trouser, 2: Pullover, 3: Dress, 4: Coat, 5: Sandal, 6: Shirt, 7: Sneaker, 8: Bag, 9: Ankle boot.

\item
Blood-MNIST~\cite{ref_medmnist}\cite{ref_blood1} 0: basophil, 1: eosinophil, 2: erythroblast, 3: immature granulocytes, 4: lymphocyte, 5: monocyte, 6: neutrophil, 7: platelet.

\item
CIFAR-10~\cite{ref_cifar10}
0: airplane, 1: automobile, 2: bird, 3: cat, 4: deer, 5: dog, 6: frog, 7: horse, 8: ship, 9: truck.

\item
Chest X-ray~\cite{ref_covid_chest_original} 0: COVID-19, 1: normal, 2: pneumonia.

\item 
Dog vs. Cat~\cite{ref_dog_cat} 0: cat, 1: dog.
\end{itemize}


\section*{Acknowledgments}
This work was supported by JSPS Grant-in-Aid for Scientific Research (C) Grant Number 23K11310.

\printcredits


\newpage
\end{document}